%% file: final_revised.tex
\documentclass[journal,twoside,web]{IEEEtran}
 \usepackage{url}

\usepackage[caption=false,font=footnotesize,labelfont=sf,textfont=sf]{subfig}

\usepackage{graphicx}
\usepackage{titling}
\usepackage{epsfig}
\usepackage{amsmath,siunitx}
\usepackage{units}
\usepackage{amssymb}
\usepackage{algorithm}
\usepackage{algpseudocode}
\usepackage{booktabs}
\usepackage{setspace}
\usepackage{nth}
\usepackage{xfrac}
\usepackage{dblfloatfix}
\usepackage{lipsum}
\usepackage{multicol}
\usepackage[export]{adjustbox}
\usepackage{textcomp}
\usepackage{graphicx}  
\usepackage{booktabs}
\input{math_commands.tex}

\renewcommand{\textrightarrow}{$\rightarrow$}
\usepackage{pdfpages}

\usepackage{amsthm}
\usepackage[utf8]{inputenc}
\usepackage[english]{babel}
\usepackage{soul}
 
\newtheorem{lemma}{Lemma}

\makeatletter

\def\rvz{{\mathbf{z}}}
\def\rvx{\mathbf{x}}

\usepackage{algcompatible}
\usepackage[noadjust]{cite}

\newcommand\blfootnote[1]{%
  \begingroup
  \renewcommand\thefootnote{}\footnote{#1}%
  \addtocounter{footnote}{-1}%
  \endgroup
}
\DeclareMathOperator*{\argminB}{argmin}

\makeatletter
\newcounter{phase}[algorithm]
\newlength{\phaserulewidth}
\newcommand{\setphaserulewidth}{\setlength{\phaserulewidth}}
\newcommand{\phase}[1]{%
  \vspace{-1.9ex}
  \Statex\leavevmode\llap{\rule{\dimexpr\labelwidth+\labelsep}{\phaserulewidth}}\rule{\linewidth}{\phaserulewidth}
  \Statex\strut\refstepcounter{phase}\textbf{#1}
  \vspace{-1.9ex}\Statex\leavevmode\llap{\rule{\dimexpr\labelwidth+\labelsep}{\phaserulewidth}}\rule{\linewidth}{\phaserulewidth}}
\makeatother

\setphaserulewidth{.35pt}

\algrenewcommand\algorithmicindent{0.9em}
\usepackage{setspace}

\newcommand{\cc}{\textcolor{black}}
\newcommand{\bb}{\textcolor{black}}

\usepackage{array}

\usepackage{url}

\begin{document}
%
\title{Target-Independent Domain Adaptation for WBC Classification using Generative Latent Search}
 \author{
Prashant Pandey, Prathosh AP, Vinay Kyatham, Deepak Mishra and 
Tathagato Rai Dastidar
}
\date{}
\markboth{IEEE Transactions on Medical Imaging}%
{Prashant \MakeLowercase{\textit{et al.}}: Target-independent Domain Adaptation for WBC Classification using Generative Latent Search}
\maketitle

\setlength{\dbltextfloatsep}{0.3cm}
\begin{abstract}

Automating the classification of camera-obtained microscopic images of White Blood Cells (WBCs) and related cell subtypes has assumed importance since it aids the laborious manual process of review and diagnosis. Several State-Of-The-Art (SOTA) methods developed using Deep Convolutional Neural Networks suffer from the problem of domain shift - severe performance degradation  when they are tested on data (target) obtained in a setting different from that of the training (source). The change in the target data might be caused by factors such as differences in camera/microscope types, lenses, lighting-conditions etc. This problem can potentially be solved using Unsupervised Domain Adaptation (UDA) techniques albeit standard algorithms presuppose the existence of a sufficient amount of unlabelled target data which is not always the case with medical images. In this paper, we propose a method for UDA that is devoid of the need for target data.
Given a test image from the target data, we obtain its `closest-clone' from the source data that is used as a proxy in the classifier. We prove the existence of such a clone given that infinite number of data points can be sampled from the source distribution.
We propose a method in which a latent-variable generative model based on variational inference is used to simultaneously sample and find the `closest-clone' from the source distribution through an optimization procedure in the latent space. \bb{We demonstrate the efficacy of the proposed method over several SOTA UDA methods for WBC classification on datasets captured using different imaging modalities under multiple settings.}
\blfootnote
{Copyright (c) 2019 IEEE. Personal use of this material is permitted. However, permission to use this material for any other purposes must be obtained from the IEEE by sending a request to pubs-permissions@ieee.org. 

Prashant Pandey and Prathosh AP are with Department of Electrical Engineering, Indian Institute of Technology Delhi, New Delhi 110016, India. Deepak Mishra is with Department of Computer Science and Engineering, Indian Institute of Technology Jodhpur, Rajasthan 342037, India. Vinay Kyatham and Tathagato Rai Dastidar are with SigTuple Technologies Pvt. Ltd., Bangalore 560102, India. Email: getprashant57@gmail.com, prathoshap@iitd.ac.in, dmishra@iitj.ac.in, trd@sigtuple.com, vinay.k@sigtuple.com. \bb{The code for our implementation is available at https://github.com/prinshul/WBC-Classification-UDA.}
}

\end{abstract}

\begin{IEEEkeywords} 
WBC, Microscopic imaging, Unsupervised domain adaptation, Generative models, VAE.
\end{IEEEkeywords}
%
\IEEEpeerreviewmaketitle
\section{Introduction}
\subsection{Background}
\IEEEPARstart{M}{icroscopic} review of Peripheral Blood Smear (PBS) slides by clinical pathologists is considered as the gold standard for detection of various disorders~\cite{bloodbook}. {This requires manual counting and classification of various types of cells, including White Blood Cells} (WBCs or leukocytes) and analysing their morphological characteristics in PBS slides.
The presence, absence, or relative counts of these cells help in the diagnosis of several types of diseases, including different forms of blood cancer, anaemia, and presence of parasites like in malaria.
This process of manual review is both laborious and error prone.
In addition, due to variations in stain, smearing process, the differentiation between various subclasses of cells is often blurry.
It takes significant expertise and experience to correctly classify all types of cells.
Lack of qualified medical professionals, especially in non-urban areas of developing countries, accentuates the problem.
Furthermore, the misdiagnosis, often caused by lack of adequate time to examine a slide thoroughly, can even lead to fatalities.
Thus, automating and standardising this process is a pressing need.

Several attempts have been made to automate some of these manual processes using methods ranging from classical computer vision~\cite{lee2013performance,young1972classification,bikhet2000segmentation} to image cytometry~\cite{hagwood2011evaluation, lippeveld2019classification}. While classical vision techniques suffer from issues like poor-generalization, image cytometry is limited by its operational speed and inability to engineer complex features~\cite{chen2016deep}. 
An alternative is to harness the power of Deep Convolutional Neural Networks (CNNs) in addressing some of these issues~\cite{qin2018fine}. In SC-CNN~\cite{sirinukunwattana2016locality}, a weighted sum of multiple classifiers is used to predict the class label of cell nuclei detected with a Spatially Constrained CNN. In \cite{mahmood2019deep}, a Conditional Generative Adversarial Network (cGAN)~\cite{mirza2014conditional} is used for nuclei segmentation, a fundamental task for cell classification. MGCNN~\cite{huang2019blood} is a White Blood Cells classification framework that combines modulated Gabor wavelet~\cite{lee1996image} and deep CNN kernels. A few commercial products too have been built utilizing some of these techniques. CellaVision~\cite{cellavision}, Shonit~\cite{shonit}, etc., automate the counting and classification of leukocytes and other blood cells. These systems consist of an automated microscope equipped with a digital camera, which captures the images of a biological sample on a glass slide. A software based analysis system, built using CNN models, is then used to localise and classify different types of cells in the sample.

\subsection{Motivation and Problem setting}

Even though the aforementioned models and systems are effective in their own ways, they suffer from certain issues that may limit their utility. For instance, Deep CNN models used for microscopic image classification are typically trained using proprietary datasets. These datasets tend to be homogeneous in terms of the capture device -- microscopes, lens and cameras used. This homogeneity and limited number of images in the training dataset cause the models trained on them to over-fit on specific characteristics of the image capturing device. As a result, when images captured with a different device or camera are presented to these models, they often wrongly classify new images, even though the trained human observers will have no difficulty in classification (images shown in Figure \ref{fig:wbctypes}). Hence, as the image capturing device changes, these models fail to adapt to the new input data distribution. This is known as the domain shift problem. Domain shift also occurs when the underlying imaging modality itself changes. For instance, a \cc{deep learning} model trained on Flow Cytometry images \cite{lippeveld2019classification} will not readily generalize for microscopic PBS images even though both capture WBCs. The problem of domain shift exists not only for medical images, but for any \cc{deep learning} system trained with single image source \cite{tzeng2017adversarial}.
\par A natural solution to this problem is to (re)-train the model with large amount of data obtained from the new device. However, generating sufficient quantity of annotated medical data is a time consuming and costly process. In addition, bottlenecks such as regulatory clearances, cause a large development cycle and delay in building such systems. We consider one such problem in this paper, where performance of CNNs trained on a dataset from a single source camera for automatic classification of images of WBCs taken from PBS, degrade when tested on unseen target dataset collected from different cameras. This falls within the ambit of a well-known computer vision problem known as \cc{Unsupervised Domain Adaptation} (UDA). \bb{However, almost all the SOTA methods on UDA~\cite{tzeng2017adversarial, ganin2017domain, sankaranarayanan2018generate} need access to the unlabelled target data during the time of training. While it may be feasible to obtain unlabelled target data, retraining of the UDA model for every newly emerging target domain might be infeasible, post their deployment in the field. Therefore, an unsupervised domain adaptation method that can operate without target data is desirable \cite{pandey2020skin}.} Motivated by these observations, in this paper we propose a UDA technique for WBC classification with following core contributions:

\begin{enumerate}
 
    \item We propose a UDA technique that does not require access to the target data \bb{during the time of training.}   
    \item We cast the problem of UDA as finding the `closest-clone' in the source domain for a given target image that is used as a proxy for the target image in the classifier trained on the source data.
    \item We theoretically prove the existence of the `closest-clone' given that infinite data points can be sampled from the source distribution. 
    \item We propose an optimization method over the latent space of variational inference based Deep generative model, to find the aforementioned clone through implicit sampling. 
    \item We demonstrate through extensive experimentation, the efficacy of the proposed method over several state-of-the-art UDA techniques for WBC classification \bb{on several datasets obtained using different imaging modalities with multiple domain shifts. We also validate our algorithm on the standard datasets used for UDA.}
\end{enumerate}

\section{Related work}
\cc{Unsupervised Domain Adaptation} (UDA) refers to the design of techniques aimed at improving the performance of machine learning tasks such as classification and segmentation when the classifier is trained using labels only from a source domain and tested on data \cc{from related} but a shifted target domain. In this section, we present a review of the state-of-the-art UDA techniques based on their principle of operation and their use in the medical imaging community.

\subsubsection{Adversarial-learning} These methods~\cite{tzeng2017adversarial, ganin2017domain, sankaranarayanan2018generate} learn domain-invariant representations using the principles of adversarial learning. 
ADDA~\cite{tzeng2017adversarial} employs a source network, pre-trained with labeled source data. Adversarial adaptation is performed by learning a target network such that a domain discriminator fails to predict the domain labels of the source and target features. During inference, the target images are mapped to the shared feature space by using the target network which are predicted by the source classifier. Generate To Adapt (GTA)~\cite{sankaranarayanan2018generate} learns domain invariant embeddings using a joint generative-discriminative set-up. During training, a feature extraction network outputs embeddings that are used by label prediction network for classification with a \cc{Generative Adversarial Network} (GAN) framework to generate realistic source images. DIRT-T~\cite{shu2018dirt} employs a Virtual Adversarial Domain Adaptation (VADA) model that pushes the decision boundaries away from regions of high data density by penalizing violation of the cluster assumption in the target domain. Transferable Adversarial Training (TAT)~\cite{liu2019transferable} generates transferable examples to fill in the gap between the source and target domains without distorting feature distributions. Domain Agnostic Learning (DAL)~\cite{peng2019domain} uses Deep Adversarial Disentangled Auto-Encoders (DADA) to disentangle domain-invariant features in the latent space by minimizing  the mutual information between domain-invariant and domain-specific features. The principles of adversarial feature learning has been used in \cite{mahmood2018unsupervised,gadermayr2019generative} to transform real images to a synthetic-like representation using unlabeled  synthetic endoscopy images and achieve stain independence. In~\cite{bertinetto2016fully}, a siamese architecture with adversarial training is used to improve the classification performance of target prostate histopathology whole-slide images. Zhang et al.~\cite{zhang2019noise} used adversarial learning for a noise adaptation task that allows a trained model to work effectively for medical images with different noise patterns.
\subsubsection{Target Reconstruction}These approaches for UDA reconstructs source or target samples as an auxiliary task that simultaneously focuses on creating a shared representation between the two domains while keeping the individual characteristics of each domain intact.
CyCADA~\cite{hoffman2017cycada} adapts between domains by aligning 
both generative and latent space
representations, with cycle and semantic consistency loss. PixelDA~\cite{bousmalis2017unsupervised} learns transformation in the pixel space from one domain to the other using task-specific and content–similarity losses. SBADA-GAN~\cite{russo2018source} maps source samples into the target domain and vice versa by imposing a class consistency loss to improve the quality of reconstructed images. I2I Adapt~\cite{murez2018image} is a framework that learns from the source domain and adapt to the target domain by extraction of domain agnostic features, domain specific reconstruction with cycle consistency losses. Tulder et al.~\cite{van2018learning} proposed a representation learning method that transforms data from different sources to a shared feature representation using per-feature normalization, a cross-modality based objective function. Goetz et al.~\cite{goetz2015dalsa} used domain adaptation to correct the sampling bias introduced with sparsely labeled MR images for tissue classification.
\subsubsection{Divergence Minimization} In these methods, source and target distributions are aligned by minimizing a divergence measure between the two distributions. Joint Adaptation Networks (JAN)~\cite{long2017deep} learns a transfer network by aligning the joint distributions of multiple domain-specific layers across domains based on a \cc{Joint Maximum Mean Discrepancy} (JMMD) criterion. Maximum Classifier Discrepancy (MCD)~\cite{saito2018maximum} aligns distributions of source and target by utilizing the task-specific decision boundaries. Task-specific classifiers are trained to detect the target samples that are far from the support of the source.
Contrastive Adaptation Network (CAN)~\cite{kang2019contrastive} estimates
the underlying label hypothesis of target samples through clustering and adapts the feature representations according to the  Contrastive Domain Discrepancy (CDD) metric. Pacheco et al.~\cite{pacheco2019unsupervised} addressed the discrepancies related to the stem cell differentiation process by minimizing a \cc{Maximum Mean Discrepancy} (MMD) based loss function in a Recurrent Neural Network (RNN) classifier.
\subsubsection{Domain Randomization}
\bb{Domain Randomization~\cite{tobin2017domain} (DR) is another class of methods related to UDA that are used to improve the generalization of classifiers. The idea is to reduce the domain shift by randomizing properties in the training environment (like source domain). Every data point in the source domain is perturbed randomly during training while assigning the same ground truth to the perturbed samples.  In methods such as~\cite{mahmood2018deep}, cinematically rendered source domain images are varied in color and texture.
For RGB images, such transformations can be obtained by varying hue, saturation, contrast and brightness.
In~\cite{toth2018training}, source images intensity is divided into multiple non-overlapping ranges. A random perturbation is added to the start/end pixel values by sampling from a Gaussian distribution. Finally, one of the following randomisation is applied to each range, a) shift the intensity values by adding a random value from a uniform distribution or b) transform the intensity values using cumulative distribution function of beta distribution or
c) simply invert the intensity range. \cite{zakharov2019deceptionnet} varies source images background color, add uniform noise, change the illumination and distort source images with different scaling factors.}

\section{Proposed Method}
\subsection{Motivation}
All the UDA methods mentioned in the previous section assumes that one has access to images from the target distribution. These images are either used to retrain the original classifier in a domain-invariant way \cite{tzeng2017adversarial,ganin2017domain,sankaranarayanan2018generate} or to align the target distribution to the source distribution \cite{hoffman2017cycada,bousmalis2017unsupervised,long2017deep,kang2019contrastive}. Also, in most of the methods \cite{tzeng2017adversarial,ganin2017domain,sankaranarayanan2018generate,kang2019contrastive}, the original classifier trained on the source data is altered, so that a new decision boundary is learned using the images from the target data in an unsupervised manner. However, in many practical situations, such as the current one, there would neither be access to the target data nor the scope to retrain the classifier. Further, a new unseen target domain may arise in the field which was not used during adaptation. 
\par We propose to address these issues in this paper by first assuming that the classifier learned on the source data (Oracle classifier) will perform well as long as the data comes from the source distribution. Subsequently, (i) we learn to sample from the source distribution and (ii) given an image from the target distribution, we find an image from the source distribution that is arbitrarily close (`closest-clone') to the given target image, under some distance metric. Finally, the target image is replaced with its `closest-clone' from the source distribution before its class is inferred by the Oracle classifier.

\setlength{\dbltextfloatsep}{0.2cm}
\begin{figure*}
\includegraphics[width=1.0\textwidth,height=.358\textwidth]{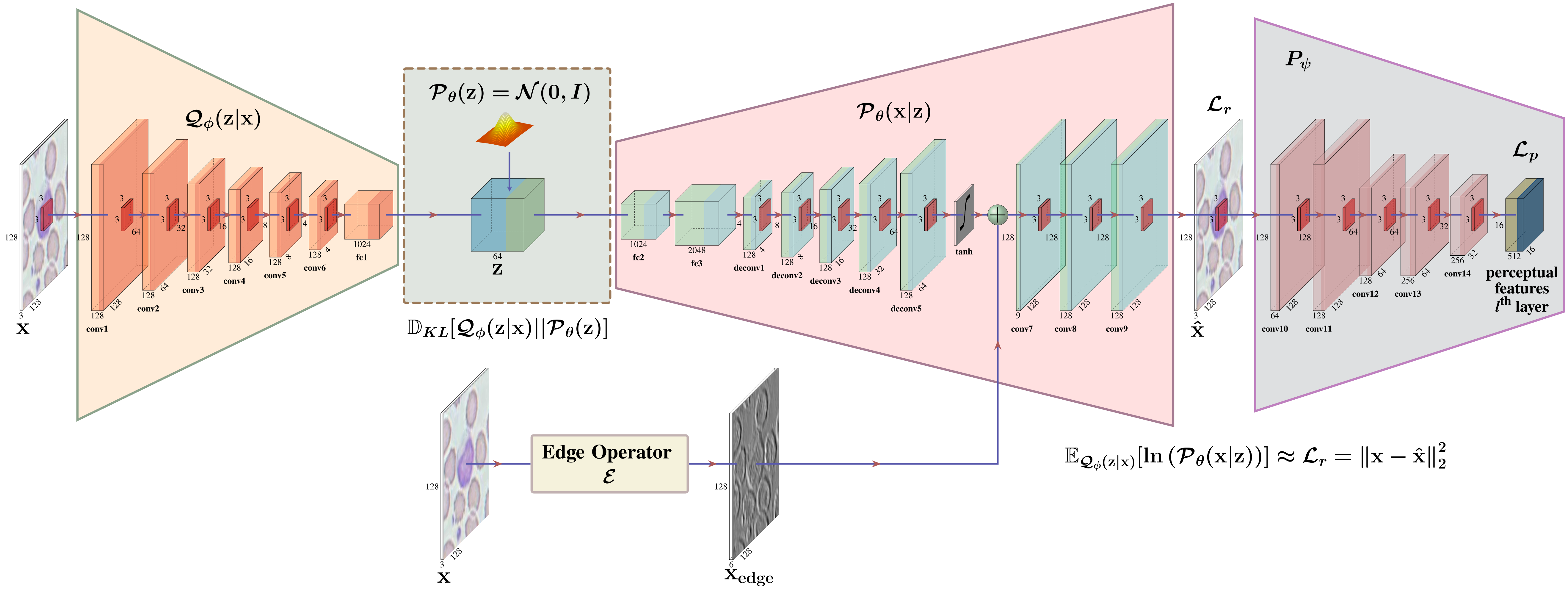}
\caption{\cc{The architecture for the Variational Auto-Encoder in the proposed method (TIGDA). Edges of the input microscopic image is concatenated with the features from the decoder $h_{\theta}$. The encoder and decoder parameters $\phi$, $\theta$ are optimized with reconstruction loss $\mathcal{L}_{r}$, KL-divergence loss $\mathbb{D}_{KL}$ and the perceptual loss $\mathcal{L}_{p}$. The perceptual model ${P}_{\psi}$ outputs $l^{\text{th}}$ layer features of VGG-16 (or ResNet-50) classifier trained on source data. A zero mean and unit variance isotropic Gaussian prior is imposed over the latent space $\bm{\mathbf{z}}$.}}
\label{fig:vaetrain}
\end{figure*}
\subsection{Existence of closest source `clone' }
To begin with, we prove that given an image from the target distribution, there exists an arbitrarily close image in the source distribution (named as `closest-clone'), provided infinite data can be sampled from the source distribution \cite{cover1967nearest}. \par Let $\mathcal{P}_s(\mathbf{x})$ and $\mathcal{P}_t(\mathbf{x})$ denote the source and the target distributions, respectively. We assume that the the underlying random variable on which $\mathcal{P}_s$  and $\mathcal{P}_t$ are defined, forms a separable metric space $\{\mathfrak{X,D}\}$ where $\mathfrak{D}$ is some distance metric. Let $\mathcal{S}_n=\{\mathbf{x}_1, \mathbf{x}_2, \mathbf{x}_3,...., \mathbf{x}_n  \}$ be i.i.d. points drawn from $\mathcal{P}_s(\mathbf{x})$ and $\tilde{\mathbf{x}}_\mathcal{T}$ be any point drawn from $\mathcal{P}_t(\mathbf{x})$. The following lemma asserts that as $n\rightarrow\infty$, there exists a point in $\mathcal{S}_n$ that it arbitrarily close to $\tilde{\mathbf{x}}_\mathcal{T}$, with probability one.

\begin{lemma}

If $\tilde{\mathbf{x}}_\mathcal{S} \in \mathcal{S}_n$ is the point such that $\mathfrak{D\{\tilde{\mathbf{x}}_\mathcal{T}},\tilde{{\mathbf{x}}}_\mathcal{S}\}< \mathfrak{D\{\tilde{\mathbf{x}}_\mathcal{T}},\mathbf{x} \} \ \forall \ \mathbf{x}\ \in\ \mathcal{S}_n $, then as ${n\to\infty}$,\  $\tilde{\mathbf{x}}_\mathcal{S}$ converges to  $\tilde{\mathbf{x}}_\mathcal{T}$ with probability  $1$ \bb{(Refer supplementary material for proof).} 
\end{lemma} 

Lemma 1 guarantees that given an image from the target distribution, an image from the source distribution, that is arbitrarily close to the given target image can be found out given the following requirements are met:
 \begin{itemize}
    \item Given a few images from the source distribution $\mathcal{P}_s$, one can sample infinite images from it. 
    \item Given infinite samples from $\mathcal{P}_s$, it is possible to find the `closest-clone' (under $\mathfrak{D}$) in $\mathcal{P}_s$, to the target image $\tilde{\mathbf{x}}_\mathcal{T}$.
 \end{itemize}
To satisfy the above requirements, in subsequent sections, we employ variational inference based sampling methods on the source distribution with which one can implicitly sample and find the `closest-clone' simultaneously. 
\subsection{Variational inference for source sampling}
In variational inference based generative models \cite{kingma2013auto}, it is assumed that the data or the observed variable (in this case images from $\mathcal{P}_s$) is generated via a two step process: (i) sample from the distribution $\mathcal{P}_{\theta}(\mathbf{z})$ of an unobserved or latent variable $\mathbf{z}$, (ii) given a data point from the latent variable, sample from the conditional distribution $\mathcal{P}_\theta(\mathbf{x|z})$ to obtain the data. Owing to the fact that the parameters of the true latent prior $\mathcal{P}_{\theta}(\mathbf{z})$ and data conditional $\mathcal{P}_\theta(\mathbf{x|z})$ are unknown, and  the posterior $\mathcal{P}_\theta(\mathbf{z|x})$ is intractable, a variational distribution, $\mathcal{Q}_\phi(\mathbf{z|x})$ is used to approximate the true posterior. With this, it can be shown that the log-likelihood of the observed data will decompose into two terms (Eq. \ref{elbo1}), an irreducible non-negative KL-divergence between  $\mathcal{P}_\theta(\mathbf{z|x})$ and  $\mathcal{Q}_\phi(\mathbf{z|x})$ and the Evidence Lower Bound (ELBO) given by Eq. \ref{elbo}. 

\begin{equation}
 \label{elbo1}
\ln\mathcal{P}_{\theta}(\mathbf{x})=\mathcal{L(\theta,\phi)}+\mathbb{D}_{KL}[\mathcal{Q}_{\phi}(\mathbf{z}|\mathbf{x})||\mathcal{P}_{\theta}(\mathbf{z|x})]
\end{equation} Here, $\mathcal{L(\theta,\phi)}$ represents ELBO which is given by,
\begin{equation}
\mathcal{L(\theta,\phi)=\mathbb{E}}_{\mathcal{Q}_{\phi}(\mathbf{z}|\mathbf{x})}[\ln\left(\mathcal{P}_{\theta}(\mathbf{x|}\mathbf{z})\right)]-\mathbb{D}_{KL}[\mathcal{Q}_{\phi}(\mathbf{z}|\mathbf{x})||\mathcal{P}_{\theta}(\mathbf{z})]
    \label{elbo}
\end{equation}

In Eq. \ref{elbo1}, the KL-term is irreducible and non-negative and thus,  $\mathcal{L(\theta,\phi)}$ serves as a lower bound on the data log-likelihood which is optimized. In deep generative model frameworks, $\mathcal{Q}_{\phi}(\mathbf{z|x})$ and $\mathcal{P}_{\theta}(\mathbf{x|z})$ are parameterized using probabilistic encoder $g_\phi$ (that outputs the parameters $\mu_\rvz$ and $\sigma_\rvz$ of a distribution)  and decoder $h_\theta$ neural networks with parameters $\phi$ and $\theta$ respectively, that maps the data space into latent space and vice-versa. Additionally,  $\mathcal{P}_{\theta}(\mathbf{z})$ is taken to be an arbitrary prior on $\mathbf{z}$ which is usually a $0$ mean and unit variance Gaussian distribution. The first term in Eq. \ref{elbo} is approximated using a norm-based divergence metric between the input and the output of the decoder as below:

\cc{
\begin{equation}
\mathbb{E}_{\mathcal{Q}_{\phi}(\mathbf{z}|\mathbf{x})}[\ln\left(\mathcal{P}_{\theta}(\mathbf{x|}\mathbf{z})\right)]\approx\mathcal{L}_r=\left\Vert \rvx - \hat{\rvx}\right\Vert_2^2
    \label{reconst}
\end{equation}}

Note that Eq. \ref{reconst} can be seen as `reconstruction' or `Auto-Encoding' of the data. Further, the second term in ELBO employs a variational approximation to the true posterior $\mathcal{P}_\theta(\mathbf{z|x})$. Thus, the aforementioned method is famously referred to as the Variational Auto-Encoder (VAE) \cite{kingma2013auto}. 
For the current problem of interest, a VAE is trained using the images from the source distribution $\mathcal{S}_n$ and once trained,  the decoder network serves as a sampler for the source distribution using a two step process: (i) sample $\mathbf{z}\sim \mathcal{N}(0,I)$, (ii) sample $\mathbf{x}$  as the output of the decoder $h_{\theta}$.

VAEs are know to produce blurred images in their conventional formulation with norm-based losses. To address this, we use the edge information (extracted using standard edge detectors) of the input image by passing it to the decoder via a skip connection, as shown in Figure~\ref{fig:vaetrain}. Rationale behind this is that unlike features such as colour and contrast, edges are in general invariant to the changes in camera characteristics. Edge information reduces the blurring due to the decoder as shown in Figure \ref{fig:ablation_imgs} and ablation studies in Table \ref{tab:ablation}. \par \bb{Further, we also incorporate the perceptual loss, which is known to enhance the generation quality of VAEs, along with the standard norm-based losses. Perceptual loss $\mathcal{L}_{p}$ between two images $\rvx$ and $\hat{\rvx}$ is defined as the Euclidean distance between the representations or the features obtained under a pre-trained classifier model $\big(P_\psi\big)$. Mathematically,
\begin{equation}
\mathcal{L}_{p} =  \left\Vert P_{\psi}( \rvx) - P_{\psi}(\hat{\rvx})\right\Vert_2^2   
\end{equation}
The idea behind $\mathcal{L}_{p}$ is that the distance metrics in a representational space learned by a classifier model trained on large scale data are better than on raw image space. This is shown to enhance image quality in several applications \cite{yang2019unsupervised}.} Figure \ref{fig:vaetrain} depicts the network diagram of the VAE on the source data with the proposed edge concatenation. 
\setlength{\textfloatsep}{0pt}

\setlength{\textfloatsep}{0pt}
\subsection{Finding `closest-clone' through Latent Search}
As mentioned in the previous sections, the objective is to simultaneously sample and search for the `closest-clone' in the source distribution, given a sample from target distribution.  Suppose a VAE has been trained on the source distribution $\mathcal{P}_s(\mathbf{x})$, the decoder $h_\theta$ of which outputs a `de-novo' image from $\mathcal{P}_s(\mathbf{x})$ by taking a normally distributed latent variable as input. That is, 
\setlength{\textfloatsep}{0pt}
\begin{equation}
\forall \ \mathbf{z}\sim \mathcal{N}(0,I), \hat{\mathbf{x}}=h_\theta(\mathbf{z})\sim \mathcal{P}_s(\hat{\mathbf{x}})
\end{equation}

Our goal is to find the `closest-clone' under some distance metric $\mathfrak{D}$, for any given image from the target distribution. Mathematically, given a $\tilde{\mathbf{x}}_\mathcal{T} \sim \mathcal{P}_t(\mathbf{x})$, find $\tilde{{\mathbf{x}}}_\mathcal{S}$ as follows:

\begin{equation}
\begin{aligned}
\tilde{{\mathbf{x}}}_\mathcal{S}= {} & \ h_\theta(\tilde{\mathbf{z}}_\mathcal{S}):\bigg \{ \mathfrak{D\{\tilde{\mathbf{x}}_\mathcal{T}},\tilde{{\mathbf{x}}}_\mathcal{S}\}< \mathfrak{D\{{\mathbf{x}}},\tilde{\mathbf{x}}_\mathcal{T} \} \\ & \ \forall \ {\mathbf{x}}\ =  h_\theta(\mathbf{z}) \sim \mathcal{P}_s({\mathbf{x}})
\label{obj}
\end{aligned}
\end{equation}

Since $\mathfrak{D}$ is computable and $h_\theta$ is a neural network that outputs a sample from $\mathcal{P}_s({\mathbf{x}})$ as a function of the latent variable $\mathbf{z}$, finding $\tilde{{\mathbf{x}}}_\mathcal{S}$ (Eq. \ref{obj}) can be cast an optimization problem over $\mathbf{z}$ with minimization of $\mathfrak{D}$ as the objective:
\begin{equation}
\tilde{\mathbf{z}}_\mathcal{S}= \argminB_\mathbf{z}\ \mathfrak{D}\big ( \tilde{\mathbf{x}}_\mathcal{T}, h_\theta(\mathbf{z}) \big )
\label{obj}
\end{equation}
\begin{equation}
\tilde{{\mathbf{x}}}_\mathcal{S}=  h_\theta(\tilde{\mathbf{z}}_\mathcal{S}) 
\end{equation}
The optimization problem is Eq. \ref{obj} can be solved using gradient descent based techniques on the decoder network $h_{\theta^{\ast}}$ $\big ( \theta^{\ast}$ are the parameters of the decoder network trained only on the source images $\mathcal{S}_n \big )$  with respect to $\mathbf{z}$. This implies that given any input image, the optimization problem in Eq. \ref{obj} will be solved to find its `closest-clone' in the source distribution which is used as a proxy in the original classifier trained only on $\mathcal{S}_n$. We call the iterative procedure of finding $\tilde{{\mathbf{x}}}_\mathcal{S}$  through optimization using $h_{\theta^{\ast}}$ as the Latent Search (LS).

Finally, inspired by the observations made  in~\cite{zhao2016loss,mishra2018ultrasound}, we propose
to use \cc{Structural Similarity Index} (SSIM) loss for $\mathfrak{D}$ to conduct the Latent Search. Unlike norm-based losses, SSIM loss helps in preservation of structural information as compared to discrete pixel level information. SSIM  is defined in~\cite{wang2004image} using the three aspects of similarities, luminance $\big(l(\mathbf{x}, \hat{\mathbf{x}})\big)$, contrast $\big(c(\mathbf{x}, \hat{\mathbf{x}})\big)$ and structure $\big(s(\mathbf{x}, \hat{\mathbf{x}})\big)$ that are measured for a pair of images $\{\mathbf{x}, \hat{\mathbf{x}}\}$ as follows:
\setlength{\textfloatsep}{0pt}
\begin{equation}
l(\mathbf{x}, \hat{\mathbf{x}}) = \frac{2\mu_\mathbf{x}\mu_{\hat{\mathbf{x}}}+C_1}{\mu_\mathbf{x}^2+ \mu_{\hat{\mathbf{x}}}^2 + C_1}
\end{equation}
\setlength{\textfloatsep}{0pt}
\begin{equation}
c(\mathbf{x}, \hat{\mathbf{x}}) = \frac{2\sigma_{\mathbf{x}}\sigma_{\hat{\mathbf{x}}}+C_2}{{\sigma_\mathbf{x}}^2+ {\sigma_{\hat{\mathbf{x}}}}^2 + C_2}
\end{equation}
\setlength{\textfloatsep}{0pt}
\begin{equation}
s(\mathbf{x}, \hat{\mathbf{x}}) = \frac{\sigma_{\mathbf{x}{\hat{\mathbf{x}}}}+C_3}{\sigma_{\mathbf{x}}\sigma_{\hat{\mathbf{x}}} + C_3}
\end{equation}
where $\mu$'s denote sample means and $\sigma$'s denote variances. $C_1, C_2$ and $C_3$ are constants as defined in \cite{wang2004image}. With these, SSIM and the corresponding loss function $\mathcal{L}_{ssim}$, for a pair of images $\{\mathbf{x}, \hat{\mathbf{x}}\}$ are defined as: 
\setlength{\textfloatsep}{0pt}
\begin{equation}
\text{SSIM}(\mathbf{x}, \hat{\mathbf{x}}) = l(\mathbf{x}, \hat{\mathbf{x}})^{\alpha} \cdot c(\mathbf{x}, \hat{\mathbf{x}})^{\beta} \cdot s(\mathbf{x}, \hat{\mathbf{x}})^{\gamma}  
\end{equation}
where $\alpha>0$, $\beta>0$ and $\gamma>0$ are parameters used to adjust the relative importance of the three components.
\setlength{\textfloatsep}{0pt}
\begin{equation}
\mathcal{L}_{ssim}(\mathbf{x}, \hat{\mathbf{x}}) = 1 - \text{SSIM}(\mathbf{x}, \hat{\mathbf{x}})
\end{equation}
Since our method does not utilize target images and employs generative Latent Search, we call our method Target-Independent Generative Domain Adaptation (TIGDA). \bb{The target independence of our method refers to the fact that we do not use target data during training, unlike SOTA UDA methods.} The inference for TIGDA is shown in Figure \ref{fig:inf}.
\begin{figure}
\includegraphics[width=.48\textwidth,height=.297\textwidth]{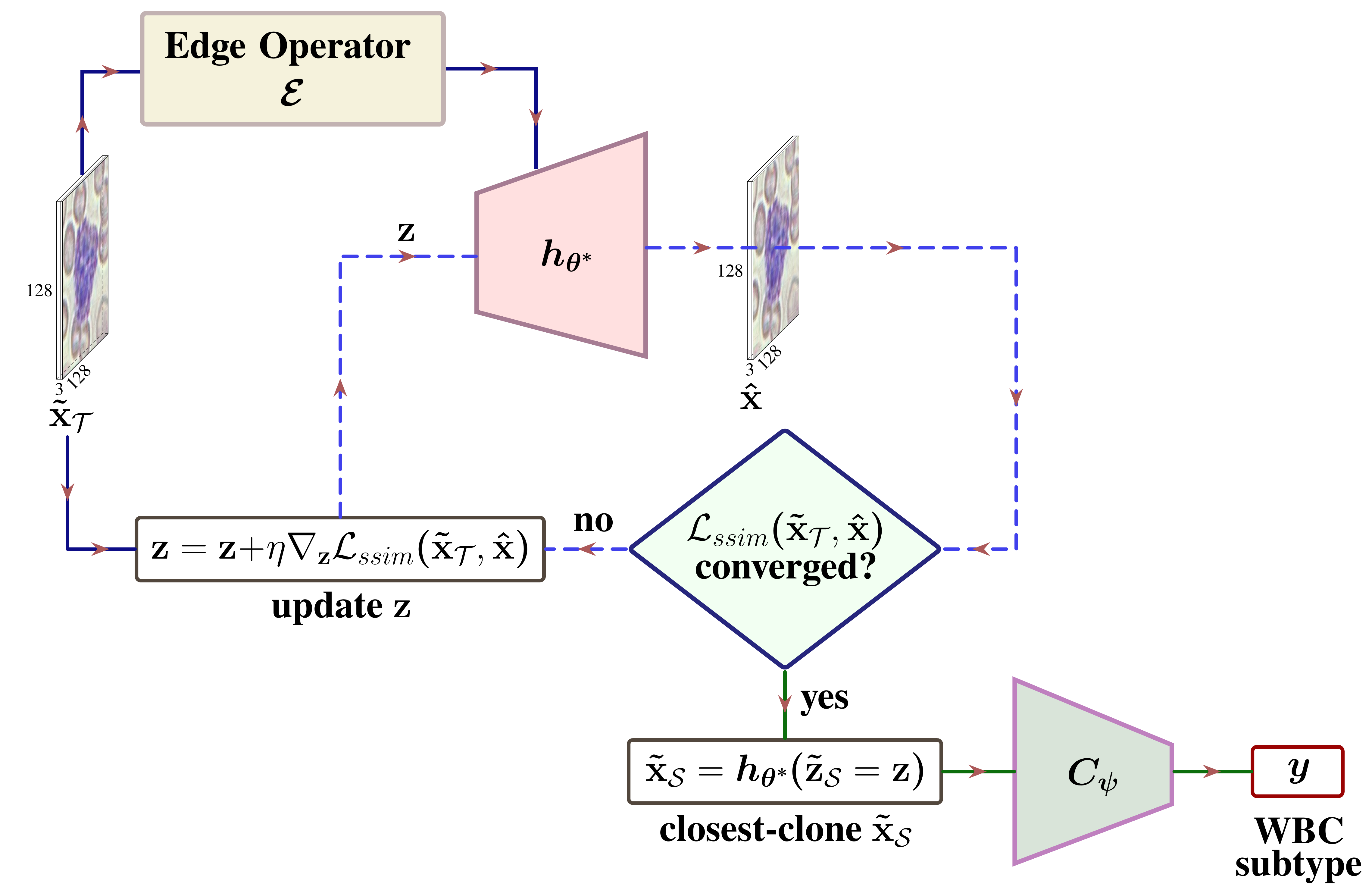}
\caption{Latent Search procedure during inference with TIGDA. The latent vector $\mathbf{z}$ is initialized with a random sample drawn from $\mathcal{N}(0,1)$. Iterations over the latent space $\rvz$ are performed to minimize the Structural Similarity loss $\mathcal{L}_{ssim}$ between the input target image $ \tilde{\mathbf{x}}_\mathcal{T}$ and the predicted target image $\mathbf{\hat x}$, which is the output of the trained decoder (blue dotted lines).  After convergence of $\mathcal{L}_{ssim}$ loss, the optimal latent vector $\bm{{\tilde{\mathbf{z}}}_{\mathcal{S}}}$, generates the `closest-clone' $ \bm{\tilde{\mathbf{x}}}_{\mathcal{S}}$ which is used to predict the class of $ \tilde{\mathbf{x}}_\mathcal{T}$ using the classifier $C_\psi$ trained on source samples.}
\label{fig:inf}
\end{figure}
\begin{figure*}
\includegraphics[width=1.0\textwidth,height=.272\textwidth]{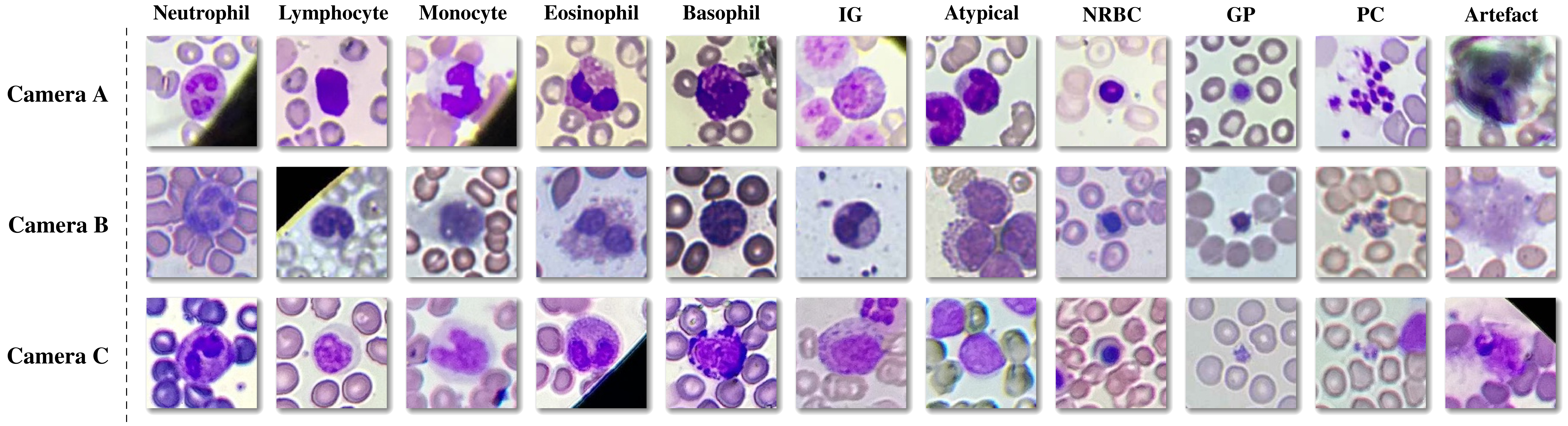}
\caption{Samples of White Blood Cells and related microscopic images (categorized into 11 classes) taken from three different cameras A, B and C. (IG=Immature granulocytes, NRBC=Nucleated red blood cells, GP=Giant platelets, PC=Platelet clumps). It is to be noted that there are no visually distinctive features across cameras but it is easy for a human-pathologist to correctly classify despite camera changes. On the other hand, \cc{deep learning} models fail to generalize across cameras.}
\label{fig:wbctypes}
\end{figure*}
\bgroup
\def\arraystretch{1.3}
\begin{table*}[hbt!]
\caption{Number of White Blood Cells and related microscopic images for each subtype (class) captured with three different cameras A, B and C. (NE=Neutrophil, LY=Lymphocyte, MO=Monocyte, EO=Eosinophil, BA=Basophil, IG=Immature granulocytes, NRBC=Nucleated red blood cells, GP=Giant platelets, PC=Platelet clumps).}
\centering
\scalebox{0.96}{
  \begin{tabular}{c|ccccccccccc|c}
    \toprule
         \textbf{Camera} & \textbf{NE} & \textbf{LY} & \textbf{MO} & \textbf{EO} & \textbf{BA} & \textbf{IG} & \textbf{Atypical} & \textbf{NRBC} & \textbf{GP} & \textbf{PC} & \textbf{Artefact} & \textbf{\scalebox{1.4}{$\rm \sfrac{\text{Train}}{\text{Test}}$}}
        \\
    \midrule
    A&3,885&1,507&2,224&2,076&65&863&984&651&486&138&2,550&\scalebox{1.4}{$\rm \sfrac{\text{10,849}}{\text{4,580}}$}\\
    B&2,045&1,840&612&373&67&1,073&2,257&97&918&796&1,437&\scalebox{1.4}{$\rm \sfrac{\text{7,997}}{\text{3,518}}$}\\
    C&85&43&144&85&12&323&861&321&303&11&16&\scalebox{1.4}{$\rm \sfrac{\text{1,548}}{\text{656}}$}\\
    \bottomrule
  \end{tabular}
  }
  \label{table:datasets}
\end{table*}
\egroup

\section {Implementation Details}
\subsection{Training of the VAE}
The Encoder $g_{\phi}$ and Decoder $h_{\theta}$ network architectures for the VAE are shown in Figure \ref{fig:vaetrain}.
We use Sobel Edge operator for Edge concatenation. Edges of the input image are concatenated with the output of \textit{tanh} nonlinearity as shown in Figure \ref{fig:vaetrain}. The VAE is trained using (i) the Mean squared error reconstruction loss $\mathcal{L}_{r}$ between the real and VAE reconstructed images and (ii) the perceptual loss $\mathcal{L}_{p}$ for which the features are taken from the $l^{\text{th}}$ layer of the VGG-16 ($10^{\text{th}}$ layer) or RestNet-50 ($38^{\text{th}}$ layer) classifier trained on source images for WBC classification task.  The hidden layers of Encoder and Decoder networks use Leaky ReLU and \textit{tanh} as activation functions with the dimensionality of the latent space being 64. VAE is trained using a standard gradient descent procedure with RMSprop optimizer.

\subsection{Inference through Latent Search}
Once the VAE is trained, given an image $\tilde{\mathbf{x}}_\mathcal{T}$ from the target distribution,
the Latent Search algorithm searches for an optimal latent vector $\tilde{\mathbf{z}}_\mathcal{S}$ that generates its `closest-clone' $\tilde{\mathbf{x}}_\mathcal{S}$ from $\mathcal{P}_S$. The search is performed by minimizing the SSIM loss $ \mathcal{L}_{ssim}$ between the input target image $\tilde{\mathbf{x}}_\mathcal{T}$ and VAE reconstructed target image. The latent vector is optimized using a gradient-based optimization procedure, performed for $K$ (a hyper-parameter) iterations over the latent space of the VAE for every target image. The gradient based optimization is implemented with Nesterov Accelerated Gradient method with a momentum of 0.5. Finally, the class for the input target image is assigned the same as the one given by the source classifier $C_{\psi}$ on $\tilde{\mathbf{x}}_\mathcal{S}$. $C_{\psi}$ is a VGG-16 or RestNet-50 classifier  trained  on  source images.  Note that our algorithm solves an optimization problem before predicting class for every input target image. However, since it involves only a forward-pass through a trained neural network (decoder $h_{\theta^{\ast}}$), the time taken is only of the order of few milliseconds on standard CPUs. \bb{The complete algorithmic steps and the architectural details for TIGDA are given in the supplementary material.}

\section{Dataset details}
The datasets used in this study will be described in this section. Peripheral blood smear (PBS) consists primarily of three cell types -- RBC (Red Blood Cell or erythrocyte), WBC (White Blood Cell or leukocyte) and platelet (or thrombocyte). Each of these primary classes have subclasses. The subclasses of WBCs are: neutrophil, lymphocyte, monocyte, eosinophil, basophil, immature granulocytes and atypical/blast cells. Apart from these, there are other types of cells and artefacts which can have appearance similar to leukocytes.
These are -- nucleated red blood cell (NRBC), large platelets, platelet clumps, and stain artefacts \cite{bloodbook}. In the current study, we consider classification of 11 categories of which seven are subtypes of WBCs and rest four are NRBC, large platelets, platelet clumps, and stain artefacts (images shown in Figure~\ref{fig:wbctypes}). 
\begin{table*}
\caption{ Accuracy (mean $\pm$ std\%) values for UDA tasks on WBC and related microscopic images captured with three different cameras A, B and C. X\textrightarrow Y indicates model trained on images from source Camera X and tested on images from target Camera Y. 
Results are reported as an average over five independent runs using various state-of-the-art UDA \bb{and Domain Randomization methods}. Note that while all UDA methods perform better than the source only model, TIGDA offers the best performance despite not using the target images.}
\begin{center}
\scalebox{0.74}{
\begin{tabular}{l|cccccc|ccccccc}
    \toprule
  
    \multicolumn{1}{c|}{} & \multicolumn{6}{c|}{ResNet-50} & \multicolumn{6}{c}{VGG-16}\\
    Models & {A\textrightarrow B} &  {A\textrightarrow C} & {B\textrightarrow A} & {B\textrightarrow C} & {C\textrightarrow A} & {C\textrightarrow B} & {A\textrightarrow B} &  {A\textrightarrow C} & {B\textrightarrow A} & {B\textrightarrow C} & {C\textrightarrow A} & {C\textrightarrow B} \\
    \midrule
    {Source Only}&42.7$\pm$0.5&51.3$\pm$0.4&35.8$\pm$0.6&46.2$\pm$0.2&22.8$\pm$0.6&26.9$\pm$0.4&37.4$\pm$0.5&47.6$\pm$0.4&31.2$\pm$0.3&40.1$\pm$0.5&17.6$\pm$0.6&22.7$\pm$0.2\\
    
   {DR1~\cite{mahmood2018deep}}&52.5$\pm$0.3&57.7$\pm$0.1&43.6$\pm$0.2&51.7$\pm$0.4&34.5$\pm$0.3&36.2$\pm$0.2&44.6$\pm$0.1&50.9$\pm$0.2&38.2$\pm$0.4&46.5$\pm$0.3&27.3$\pm$0.3&30.8$\pm$0.2\\
    
    {DR2~\cite{toth2018training}}&60.3$\pm$0.2&65.4$\pm$0.3&55.9$\pm$0.2&64.2$\pm$0.4&44.6$\pm$0.3&49.8$\pm$0.4&54.1$\pm$0.1&59.6$\pm$0.2&48.7$\pm$0.1&60.5$\pm$0.4&41.3$\pm$0.3&45.2$\pm$0.1\\
    
    {DR3~\cite{zakharov2019deceptionnet}}&50.4$\pm$0.2&53.4$\pm$0.4&40.5$\pm$0.2&49.8$\pm$0.3&29.5$\pm$0.3&32.7$\pm$0.4&41.8$\pm$0.3&47.5$\pm$0.2&35.9$\pm$0.1&42.1$\pm$0.2&23.6$\pm$0.2&28.3$\pm$0.3\\

    {ADDA~\cite{tzeng2017adversarial}}&43.5$\pm$0.1&52.7$\pm$0.2&37.3$\pm$0.1&48.1$\pm$0.5&24.9$\pm$0.4&29.1$\pm$0.5&39.3$\pm$0.2&50.1$\pm$0.3&33.6$\pm$0.4&43.3$\pm$0.2&19.8$\pm$0.4&25.2$\pm$0.5\\
   
  
   {GTA~\cite{sankaranarayanan2018generate}}&56.2$\pm$0.4&66.3$\pm$0.5&48.1$\pm$0.2&56.7$\pm$0.6&35.5$\pm$0.4&37.8$\pm$0.1&52.6$\pm$0.7&62.1$\pm$0.3&41.9$\pm$0.6&50.7$\pm$0.3&30.1$\pm$0.1&33.7$\pm$0.6\\
   

   {TAT~\cite{liu2019transferable}}&65.8$\pm$0.5&70.5$\pm$0.4&54.8$\pm$0.3&63.1$\pm$0.7&44.7$\pm$0.2&48.2$\pm$0.3&61.7$\pm$0.5&67.3$\pm$0.4&50.6$\pm$0.4&58.3$\pm$0.6&40.3$\pm$0.1&42.5$\pm$0.1\\

   {DIRT-T~\cite{shu2018dirt}}&55.7$\pm$0.5&65.1$\pm$0.6&49.2$\pm$0.2&55.4$\pm$0.3&34.2$\pm$0.3&37.5$\pm$0.4&53.1$\pm$0.8&61.9$\pm$0.7&40.7$\pm$0.5&50.3$\pm$0.5&31.3$\pm$0.4&32.9$\pm$0.7\\

    {DAL~\cite{peng2019domain}}&64.7$\pm$0.2&69.4$\pm$0.3&56.3$\pm$0.2&62.7$\pm$0.4&43.5$\pm$0.1&47.5$\pm$0.5&60.8$\pm$0.2&66.5$\pm$0.5&51.8$\pm$0.4&59.1$\pm$0.3&39.7$\pm$0.1&41.1$\pm$0.2\\
    

    {CyCADA~\cite{hoffman2017cycada}}&67.2$\pm$0.5&73.7$\pm$0.1&58.2$\pm$0.2&64.5$\pm$0.6&48.4$\pm$0.4&50.2$\pm$0.3&62.3$\pm$0.3&70.2$\pm$0.2&53.4$\pm$0.4&59.7$\pm$0.2&42.6$\pm$0.6&43.9$\pm$0.7\\
   

   {PixelDA~\cite{bousmalis2017unsupervised}}&65.9$\pm$0.2&71.8$\pm$0.7&59.1$\pm$0.8&66.2$\pm$0.5&47.8$\pm$0.4&50.6$\pm$0.5&61.5$\pm$0.3&68.4$\pm$0.4&54.6$\pm$0.7&58.8$\pm$0.6&41.3$\pm$0.6&42.5$\pm$0.4\\

    {SBADA-GAN~\cite{russo2018source}}&66.3$\pm$0.2&70.5$\pm$0.2&60.3$\pm$0.3&65.6$\pm$0.4&46.4$\pm$0.7&51.1$\pm$0.1&62.7$\pm$0.6&67.9$\pm$0.8&53.8$\pm$0.7&58.7$\pm$0.2&42.7$\pm$0.4&44.6$\pm$0.7\\

    {I2IAdapt~\cite{murez2018image}}&64.4$\pm$0.6&68.7$\pm$0.5&61.2$\pm$0.3&65.4$\pm$0.4&45.2$\pm$0.1&49.7$\pm$0.6&63.9$\pm$0.8&65.1$\pm$0.1&52.5$\pm$0.7&55.6$\pm$0.4&43.8$\pm$0.8&45.3$\pm$0.3\\

    {JAN~\cite{long2017deep}}&49.6$\pm$0.2&58.2$\pm$0.5&43.3$\pm$0.2&54.7$\pm$0.4&30.2$\pm$0.7&35.4$\pm$0.8&43.5$\pm$0.6&54.2$\pm$0.4&39.1$\pm$0.3&47.5$\pm$0.3&26.3$\pm$0.4&31.4$\pm$0.6\\

   {MCD~\cite{saito2018maximum}}&55.4$\pm$0.4&67.1$\pm$0.8&49.2$\pm$0.7&55.8$\pm$0.6&36.1$\pm$0.2&39.2$\pm$0.5&50.9$\pm$0.7&63.2$\pm$0.4&42.3$\pm$0.3&50.4$\pm$0.5&31.9$\pm$0.8&34.8$\pm$0.5\\

  {CAN~\cite{kang2019contrastive}}&67.8$\pm$0.4&71.3$\pm$0.5&63.4$\pm$0.5&65.4$\pm$0.3&47.3$\pm$0.2&51.2$\pm$0.4&61.9$\pm$0.8&68.1$\pm$0.3&54.6$\pm$0.6&59.3$\pm$0.4&40.9$\pm$0.2&45.7$\pm$0.8\\
  
  {TIGDA (Ours)}&\textbf{76.2$\pm$0.3}&\textbf{80.1$\pm$0.4}&\textbf{72.3$\pm$0.5}&\textbf{74.8$\pm$0.6}&\textbf{53.5$\pm$0.4}&\textbf{56.2$\pm$0.3}&\textbf{71.8$\pm$0.5}&\textbf{76.7$\pm$0.2}&\textbf{63.2$\pm$0.5}&\textbf{68.6$\pm$0.7}&\textbf{50.8$\pm$0.2}&\textbf{55.1$\pm$0.4}\\
    
    \bottomrule
\end{tabular}
}
\end{center}
\label{tab:comp}
\end{table*} 
\begin{figure*}
\centering
    \subfloat[ADDA]{\includegraphics[width=0.16\linewidth,height=0.14\linewidth]{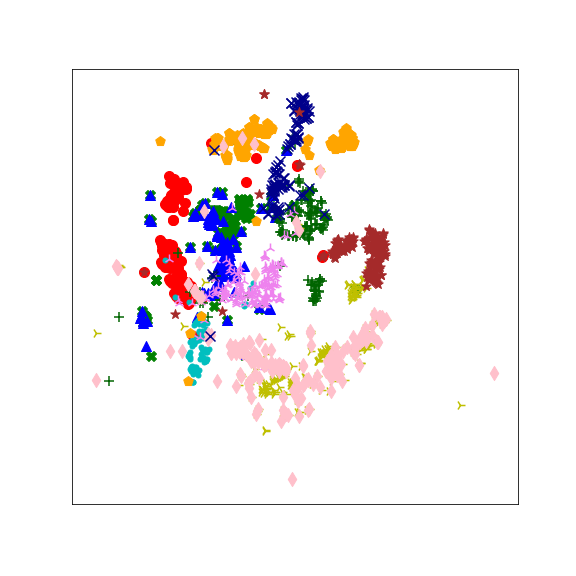}
        \label{fig:adda}}
     \subfloat[GTA]{\includegraphics[width=0.16\linewidth,height=0.14\linewidth]{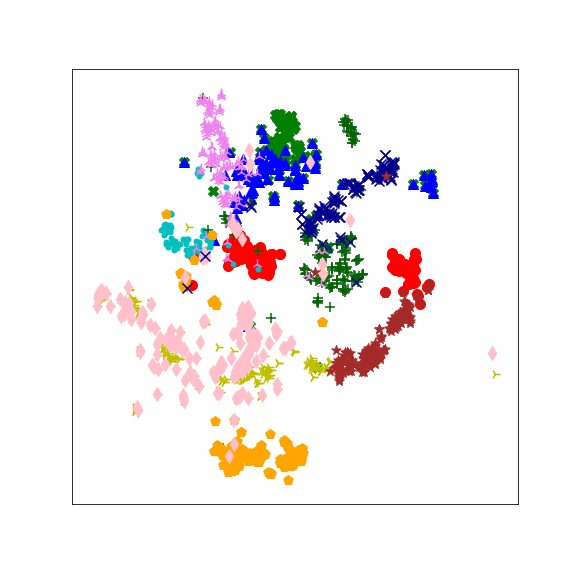}
      \label{fig:gta}}
         \subfloat[DIRT-T]{\includegraphics[width=0.16\linewidth,height=0.14\linewidth]{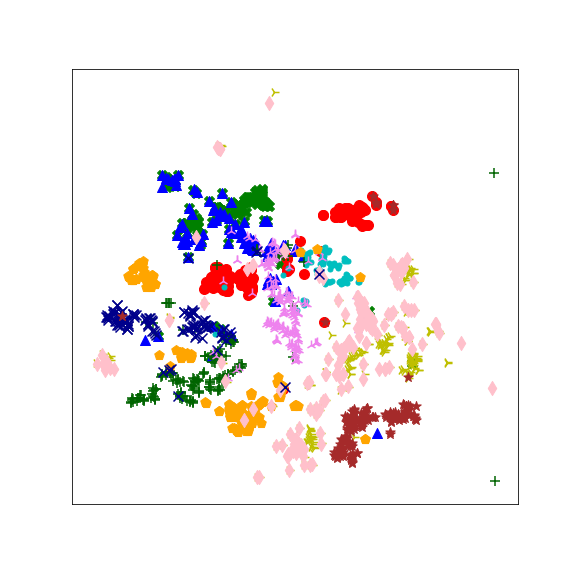}
      \label{fig:dirtt}}
    \subfloat[TAT]{\includegraphics[width=0.16\linewidth,height=0.14\linewidth]{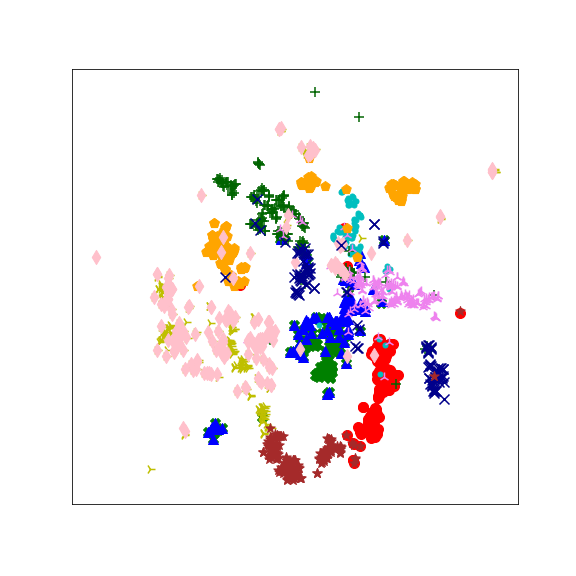}
      \label{fig:tat}} 
     \subfloat[DAL]{\includegraphics[width=0.16\linewidth,height=0.14\linewidth]{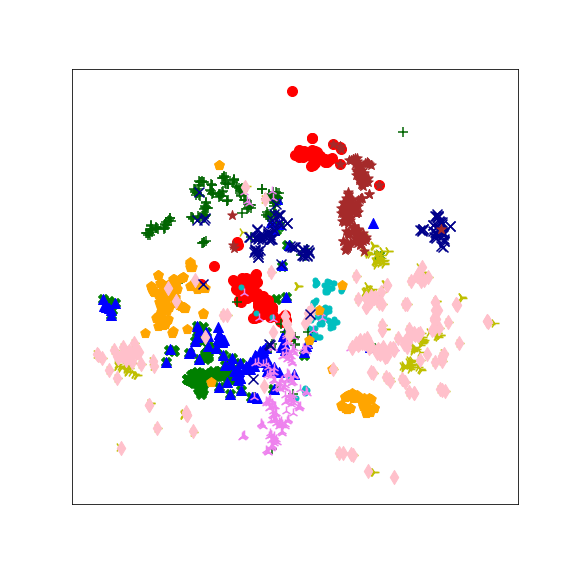}
      \label{fig:dal}} 
     \subfloat[TIGDA]{\includegraphics[width=0.16\linewidth,height=0.14\linewidth]{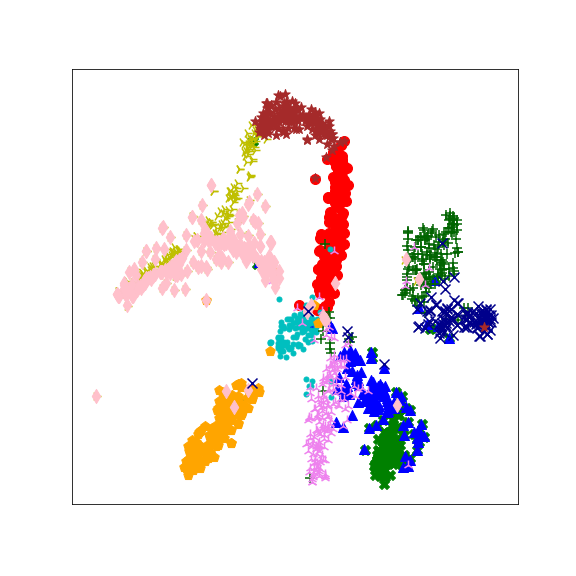}
      \label{fig:ours_tsne}} 
 \caption{\bb{t-SNE plots of features generated by ADDA~\cite{tzeng2017adversarial}, GTA~\cite{sankaranarayanan2018generate}, DIRT-T~\cite{shu2018dirt}, TAT~\cite{liu2019transferable}, DAL~\cite{peng2019domain} and TIGDA on domain adaptation task A\textrightarrow C. We used different markers and different colors to denote 11 categories. It is seen that TIGDA offers better clustering as compared to the rest.}}
\label{fig:tsne}
\end{figure*}\bb{Data used in our experiments comprises images from the PBS slides processed after complete de-identification to remove all the patient information, including age and gender. These were collected from two large clinical laboratories in Bangalore, India. The internal ethics committee of the respective laboratories approved the study. The samples were collected retrospectively 
without prospective patient recruitment.} \par\bb{The hardware consists of the following components,
(a) Optical system: Consists of an optical tube (40X or 100X Plan Achromat objective and 10X eyepiece) and Abbe Condenser with white LED source,
(b) Camera: The system is built such that either a mobile phone or a USB camera can be fitted on top of the eyepiece with a 3D printed attachment, aligning the optical axis of the tube/eyepiece with the camera,
(c) Hardware control: A small PCB designed to receive USB commands and drive motors and LED,
(d) XYZ slide stage: The XYZ platform is built using commercially available low-cost ball screws and stepper motors, along with some machined parts \cite{dastidar2020whole}.
The images used in this work are captured through 3 different cameras -- One cell phone make (iPhone 6s) and two brands of USB camera (from e-con systems~\cite{econ} and das-Cam \cite{dascam}). All cameras had resolution of at least 13MP with varying hardware and optical designs that induce the domain shift. For example, econ camera has an AR1335 CMOS image sensor and lens with 1/3.2'' optical format while das-Cam contains an OV13850 CMOS sensor with a lens of 1/3.06'' form factor.} \par Images are collected only from the `monolayer' region of the slides -- where the red blood cells are just touching each other. This is the area of the slide which is typically used for manual analysis~\cite{bloodbook}.
Slides prepared using varied staining types were used.
The images are of size approximately 13MP, with a spatial resolution of around 5.5 pixels per micron.
WBC and other similar looking cells (as described above) are localised in these images using a U-Net~\cite{unet} based technique described in~\cite{shonit}.
Each sample slides can potentially yield hundreds of unique WBC candidates. For annotation, we cropped $128\times128$ area around the WBCs identified by the extraction model.
These cells are then presented to three different certified medical professionals for annotating into different subtypes, using an in-house web based annotation tool.
There is usually a high degree (as high as 20\%) of inter-observer variability in the data annotation process.
Therefore, we use only those images where at least 2 out of 3 clinical pathologists agree on the class while the rest of the images are rejected. Table \ref{table:datasets} describes the summary of the datasets named as A, B and C corresponding to three cameras used. 
\section{Experiments and Results}
\subsection{Benchmarking Experiments}
\bb{In the first set of experiments, we benchmark performance of the baseline classifier with the following experiments: (a) Train and test on the same dataset type (A/B/C), (b) Train and test by combining images from all dataset types (A+B+C), (c) Train on one dataset and test on the other (all six combinations) with and without class balancing. The notation {X\textrightarrow Y} symbolizes training on a dataset X and testing on Y.}

\bb{Table \ref{table:sametasks} lists the results of experiment (a) which establishes an upper bound on the performance and (b) where it is seen that the performance degrades when all images from all three datasets are combined. This is due to the existence of domain shift between the datasets that makes learning difficult even with supervision. Moreover, combining datasets is not possible in the UDA setting where the labels are not known for the target data. Results of experiment (c) are shown in Table \ref{table:samesamples} where it is seen that the accuracy severely degrades when train and test sets are from different domains despite inducing an artificial class balance. The goal of UDA techniques is to improve the accuracies reported in Table \ref{table:samesamples}.} 
\begin{table}[hbt!]
\caption{\bb{Benchmarking A,B and C datasets using ResNet-50 classifier with different train and test sets. It is seen that combining all datasets makes learning difficult because of domain shift.}}
\centering
\scalebox{0.79}{
 \color{black} \begin{tabular}{ccccc}
    \toprule
         Measure & {A\textrightarrow A} & {B\textrightarrow B} & {C\textrightarrow C} & {(A+B+C)\textrightarrow (A+B+C)}
        \\
    \midrule
     Train Acc. &98.6$\pm$0.1&99.3$\pm$0.2&100.0$\pm$0.0&98.7$\pm$0.2\\
    Test Acc. &95.2$\pm$0.2&94.0$\pm$0.3&92.5$\pm$0.1&84.4$\pm$0.3\\
    \bottomrule
  \end{tabular}
  }
  \label{table:sametasks}
\end{table}
\setlength{\textfloatsep}{0pt}
\begin{table}[hbt!]
\caption{\bb{Accuracy on Resnet-50 classifiers for different Adaptation tasks. In the second row, all the three datasets are made to have same size by randomly subsampling the datasets}.}
\centering
\scalebox{0.67}{
  \color{black}\begin{tabular}{ccccccc}
    \toprule
         Measure & {A\textrightarrow B} &   {A\textrightarrow C} & {B\textrightarrow A}  & {B\textrightarrow C} & {C\textrightarrow A} & {C\textrightarrow B}
        \\
    \midrule
   W/o Balance&  42.7$\pm$0.5&51.3$\pm$0.4&35.8$\pm$0.6&46.2$\pm$0.2&22.8$\pm$0.6&26.9$\pm$0.4\\
   With Balance &  40.4$\pm$0.1&36.2$\pm$0.4&38.9$\pm$0.2&30.5$\pm$0.2&24.5$\pm$0.4&28.2$\pm$0.3\\
    
    \bottomrule
  \end{tabular}
  }
  \label{table:samesamples}
\end{table}
\subsection{Baseline Experiments}
\bb{The first set of task is of classification across 11 classes with classifiers trained on one (source) dataset and tested on another (target) dataset. We report average classification accuracies with standard-deviation (averaged over five independent runs) with two backbone architectures for the source classifier: ResNet-50 and VGG-16.} \bb{For all the UDA tasks, the VAE is trained with the entire source data and tested on the entire target data.} \bb{Table \ref{tab:comp} compares the performance of TIGDA with 12 SOTA UDA baselines, along with the accuracy without any UDA (called Source Only). It is seen that although all the UDA methods improve upon the Source Only performance, TIGDA offers the best performance despite not using any data from the target distribution. The confusion matrix for a few methods is given in the Figure 2 of the Supplementary material.  We also compare with three Domain Randomization (DR) techniques, DR1 \cite{mahmood2018deep}, DR2 \cite{toth2018training} and DR3 \cite{zakharov2019deceptionnet}. While DR provides performance boost, they have poorer performance as compared to TIGDA. This is because DR methods typically work well when the unseen target is within the scope of the class of random perturbations that are made on the source} \bb{which is not the case always. In TIGDA on the other hand, every target image is made to resemble the source image through implicit sampling. Since VAE learns to sample from the entire source domain, the domain shift is implicitly reduced during inference without explicitly assuming any form for the shift.}
\bb{It is also observed that the performance of the classifier when trained and tested on single source domain (around 92-95\% for all the datasets) do not degrade with TIGDA.} 
\setlength{\textfloatsep}{0pt}
\subsubsection{t-SNE}
\bb{To further examine our hypothesis, in Figure \ref{fig:tsne} we depict the t-SNE~\cite{maaten2008visualizing} 
plots of features generated by adversarial based UDA methods (ADDA~\cite{tzeng2017adversarial}, GTA~\cite{sankaranarayanan2018generate}, DIRT-T~\cite{shu2018dirt}, TAT~\cite{liu2019transferable} and DAL~\cite{peng2019domain}) for the domain adaptation task A\textrightarrow C. For TIGDA, we plot the embeddings of the latent variable $\tilde{\rvz}_\mathcal{S}$ obtained through the LS on the target images. It is seen that the representation generated by the LS of TIGDA is more separated compared to those generated by adversarial training based UDA methods.} \bb{A similar observation is made on the first two principal component plots of the latent representations (Please refer to Figure 1 in supplementary material).}
\setlength{\textfloatsep}{0pt}
\begin{figure}
\includegraphics[width=.48\textwidth,height=.20\textwidth]{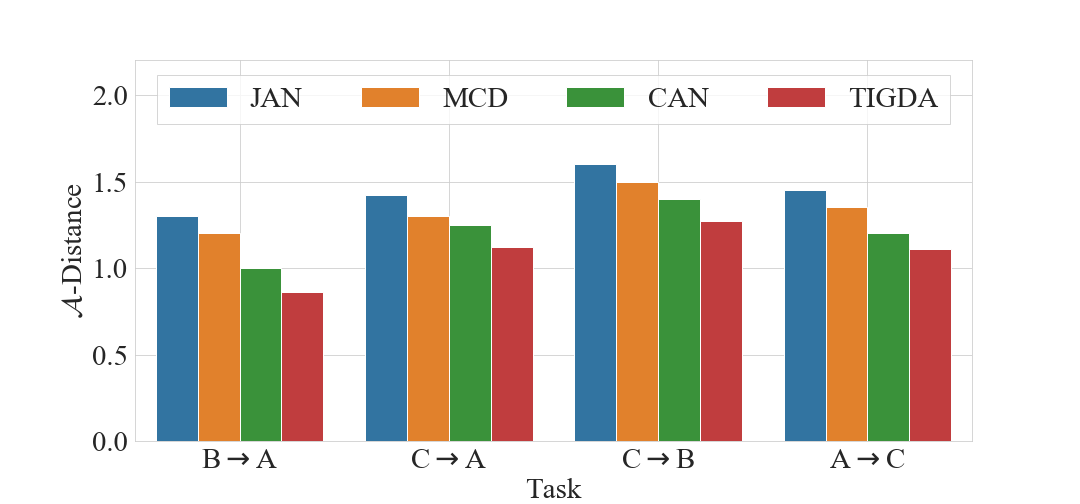}
\caption{$\mathcal{A}$-Distance (lower is better) of JAN~\cite{long2017deep}, MCD~\cite{saito2018maximum}, CAN~\cite{kang2019contrastive} and TIGDA.}
\label{fig:adist}
\end{figure}

\setlength{\textfloatsep}{0pt}
\begin{figure}
\centering
\includegraphics[width=0.48\textwidth,height=.350\textwidth]{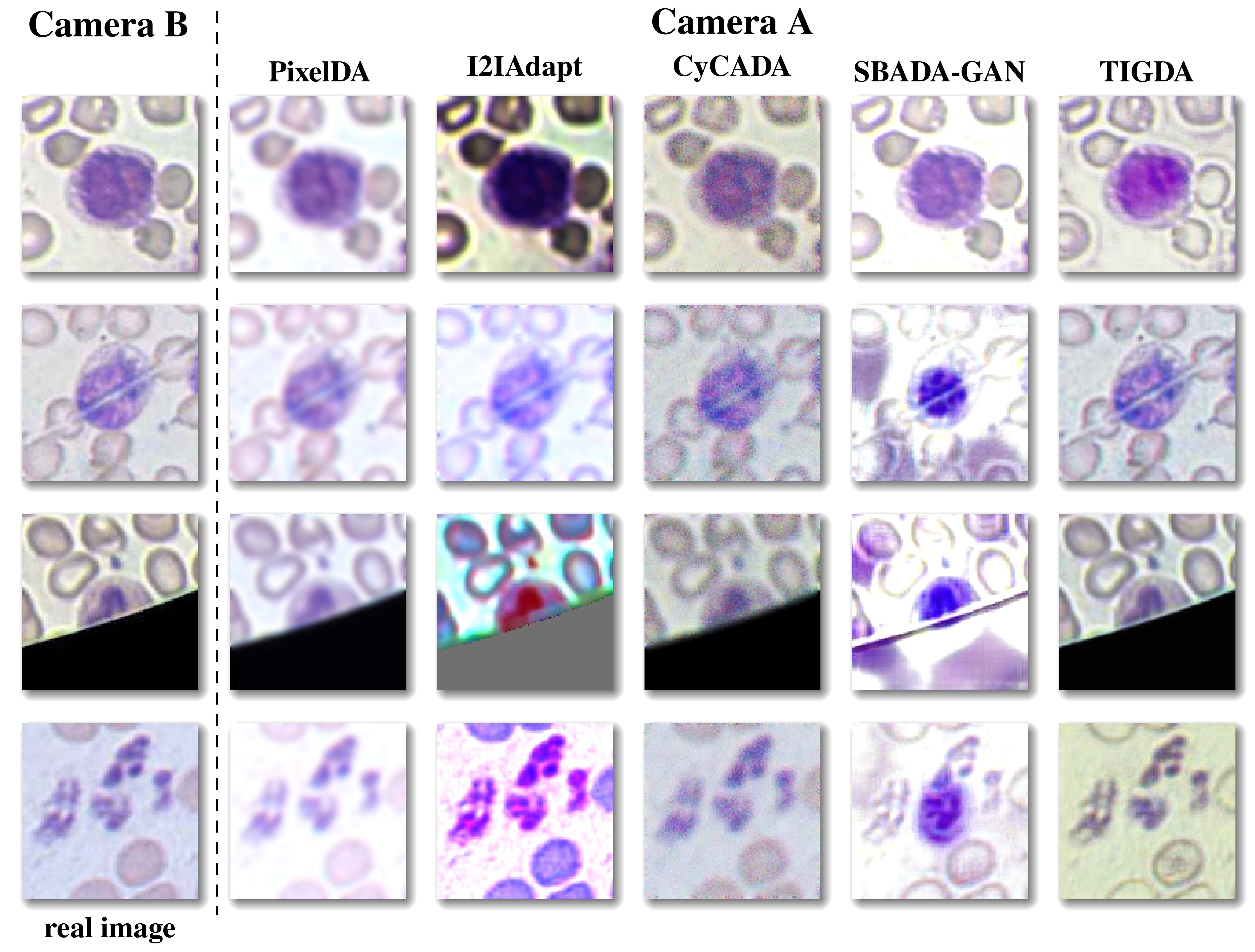}
\caption{Translation of images from one domain (Camera B) to other (Camera A) using reconstruction based domain adaptation methods: PixelDA~\cite{bousmalis2017unsupervised}, I2IAdapt~\cite{murez2018image}, CyCADA~\cite{hoffman2017cycada}, SBADA-GAN~\cite{russo2018source}. In TIGDA, we depict the `closest-clones' of Camera B (target) images in the Camera A (source) domain. It is seen that TIGDA preserves the edges, perceptual quality and structural details in the generated clones. }
\label{fig:recons_comp}
\end{figure}
\setlength{\textfloatsep}{0pt}
\begin{figure}
\includegraphics[width=0.488\textwidth,height=.265\textwidth]{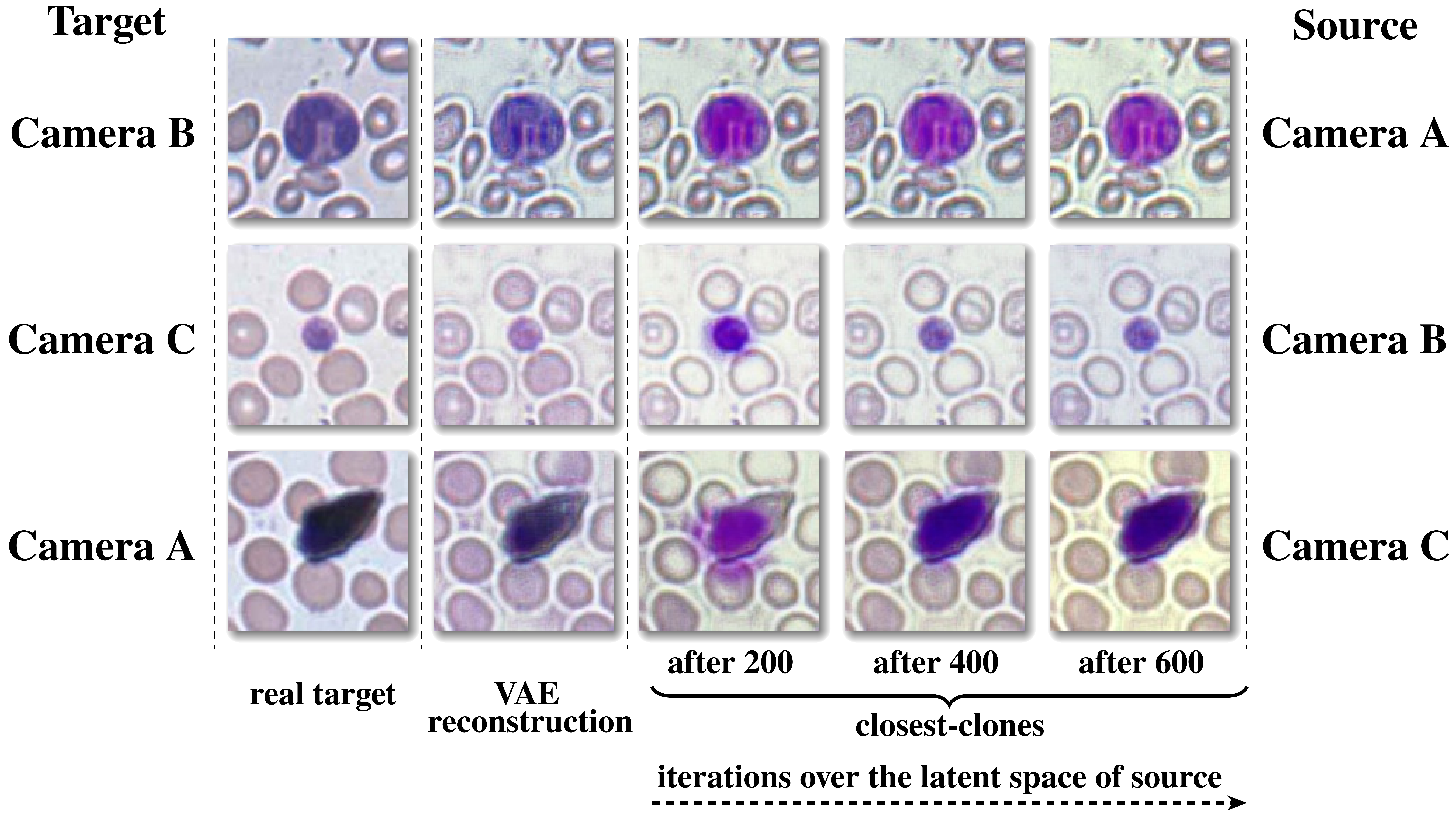}
\caption{Illustration of Latent Search in TIDGA. VAE reconstructs images prior to LS. The closest-clones obtained after every 200 iterations are shown. A transformation is observed from the target to the source domain as the LS progresses.}
\label{fig:transform}
\end{figure}
\setlength{\textfloatsep}{0pt}
\subsubsection{$\mathcal{A}$-Distance}
\bb{To ascertain the closeness of the `closest-clones' obtained through the LS, to the source distribution, we compute the $\mathcal{A}$-distance~\cite{ben2007analysis}, which is a measure of similarity between two probability distributions. Similar feature distributions will have lower $\mathcal{A}$-distance between them as compared to dissimilar feature distributions. $\mathcal{A}$-distance is given by $\hat d_A = 2(1 - 2\epsilon)$ where $\epsilon$ is the generalization error  of a linear SVM classifier trained to discriminate between the source and target domains.} \bb{Figure \ref{fig:adist} displays $\hat d_A$ for the four domain adaptation tasks with JAN~\cite{long2017deep} features, MCD~\cite{saito2018maximum} features and CAN~\cite{kang2019contrastive} features, respectively. In our case, $\hat d_A$ is measured between the latent vectors (produced by the Encoder of the VAE) of the source images and the latent vectors of the `closest-clones' for target images obtained from Latent Search. We  observe that $\hat d_A$ is smallest in our case as compared to other methods for all the tasks. This implies that the features obtained using TIGDA are transferable between the source and target domains, aiding better adaptation.}
\subsubsection{Qualitative examination}
\bb{To qualitatively examine the performance of the reconstruction-based methods, we plot the transformed target samples from (source) Camera B to (target) Camera A for different methods as shown in Figure \ref{fig:recons_comp}. It is seen that I2IAdapt~\cite{murez2018image} and SBADA-GAN~\cite{russo2018source} are not able to capture fine subtleties of partially visible White Blood Cells in microscopic images that results in poor performance. PixelDA~\cite{bousmalis2017unsupervised} and CyCADA~\cite{hoffman2017cycada} result in blurry images while TIGDA generated images are better where it is seen that the subtleties like edge information are well-preserved. In summary, we have demonstrated that TIGDA achieves better performance over the SOTA adversarial, divergence and reconstruction based UDA methods without any requirement for target images.}
\bb{\subsubsection{One-shot learning}
Even though TIGDA does not utilize the target data during training, target image is used for LS during inference. Therefore, we also compare TIGDA with SOTA one-shot learning techniques in Table \ref{tab:oneshot}. In one-shot learning methods, a single target image is used during training for adaptation. It is seen that TIGDA outperforms such techniques. This is because, in one-shot learning methods, the target image that is used for training is fixed which restricts the learnability. However in TIGDA, no target image is used during training but a fresh latent search is conducted on each input target image during inference.}
\setlength{\textfloatsep}{0pt}
\begin{table}
\caption{\bb{Comparison of TIGDA with One-shot learning methods.}}
\begin{center}
\color{black}\begin{tabular}{l|c|c}
    \toprule
    Method  & {A\textrightarrow B} &  {C\textrightarrow B} \\
    \midrule
    ProtoNet~\cite{chen2019closer}&61.9$\pm$0.1&49.6$\pm$0.3\\
     MatchingNet~\cite{chen2019closer}&57.6$\pm$0.2&43.7$\pm$0.1\\
     DAPN~\cite{zhao2020domain}&68.9$\pm$0.2&51.9$\pm$0.2\\
     DN4~\cite{li2019revisiting}&55.4$\pm$0.1&44.6$\pm$0.2\\
      FADA~\cite{motiian2017few}&60.6$\pm$0.3&45.9$\pm$0.3\\
      {TIGDA (Ours)}&76.2$\pm$0.3&56.2$\pm$0.3\\
\bottomrule
\end{tabular}
\end{center}
\label{tab:oneshot}
\end{table} 
\setlength{\textfloatsep}{0pt}
\subsection{Ablation studies}
\bb{To examine the contributions made by each of the proposed components, we conduct several ablation experiments} \bb{on TIGDA in this section.} 
\subsubsection{Effect of number of iterations on LS}
\bb{The inference of TIGDA involves a gradient-based optimization through the decoder network $h_{\theta^{\ast}}$ to generate the `closest-clone' for a given target image. In Figure \ref{fig:transform},  we show the transformation of a few target images after every 200 iterations. It can be seen that as the number of iterations increase, the target images change their characteristics to move towards the source distribution.}
\setlength{\textfloatsep}{0pt}
\begin{figure}[hbt!]
       \subfloat[Inference on camera C microscopic images when the model is trained on camera A images.]{\includegraphics[width=0.470\linewidth,height=0.41\linewidth]{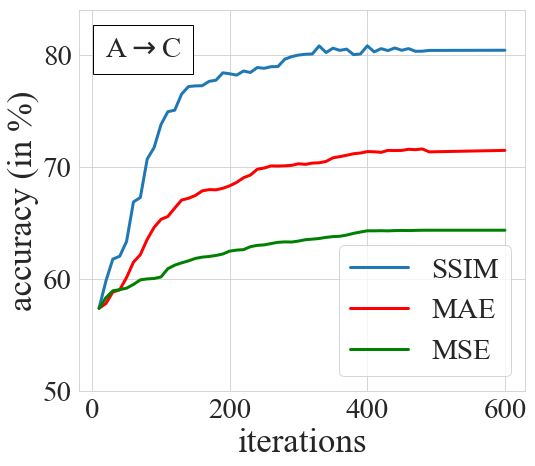}
        \label{fig:alossacc}}
    \hfill
    \hfill
       \subfloat[Inference on camera C microscopic images when the model is trained on camera B images.]{\includegraphics[width=0.470\linewidth,height=0.41\linewidth]{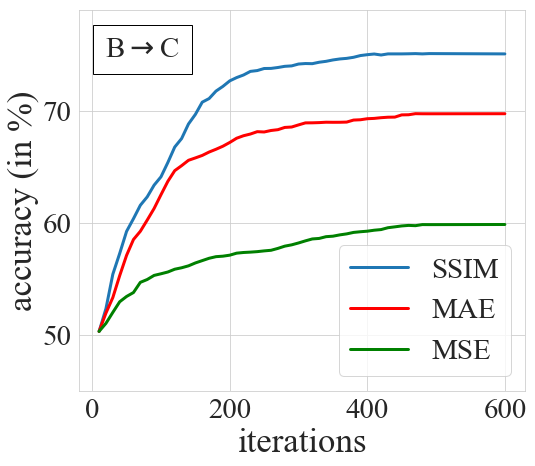}
        \label{fig:blossacc}}
    \caption{Performance of gradient-based Latent Search during inference on target microscopic images for two domain adaptation tasks using different objective functions; MSE=Mean Squared Error, MAE=Mean Absolute Error, SSIM=Structural Similarity Index. It is seen that the loss saturates around 500-600 iterations.}
\label{fig:lossacc}
\end{figure}
\bb{Quantitatively, we plot the accuracy as a function of number} of iterations in Figure \ref{fig:lossacc} where it is seen that it saturates around 500-600 iterations. We thus used 600 iterations in all the previous experiments in Table \ref{tab:comp}. 
\subsubsection{Effect of the Edge concatenation}
As described earlier, the edge-map of the input image is concatenated with one of the layers of decoder both while training and inference. Figure \ref{fig:wec} shows the quality of image generated after Latent Search when the model was trained without edge concatenation (wEc). It can be observed that edge information of the nucleus and surrounding cells is lost resulting in a blurry image. Further, the accuracy drops to 57.6\% if edge concatenation is removed from VAE for the task A\textrightarrow B as evident from Table \ref{tab:ablation}, whereas the accuracy for TIGDA is 76.2\% for the same task. Similarly, the accuracy drops to 60.3\% for the task B\textrightarrow C without edge concatenation while it is 74.8\% for TIGDA.
\setlength{\textfloatsep}{0pt}
\begin{figure}[hbt!]
\centering
\scalebox{.99}{
    \subfloat[real target]{\frame{\includegraphics[width=0.18\linewidth,height=0.18\linewidth]{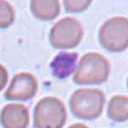}}
        \label{fig:ab_realt}}
        \hfill
         \hfill
         \subfloat[wEc]{\frame{\includegraphics[width=0.18\linewidth,height=0.18\linewidth]{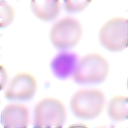}}
      \label{fig:wec}}
      \hfill
       \hfill
         \subfloat[wPl]{\frame{\includegraphics[width=0.18\linewidth,height=0.18\linewidth]{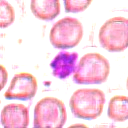}}
      \label{fig:wpl}}
      \hfill
       \hfill
    \subfloat[wLS]{\frame{\includegraphics[width=0.18\linewidth,height=0.18\linewidth]{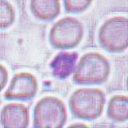}}
      \label{fig:wls}} 
      \hfill
       \hfill
     \subfloat[TIGDA]{\frame{\includegraphics[width=0.18\linewidth,height=0.18\linewidth]{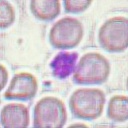}}
      \label{fig:ab_ours}} 
      }
\caption{Ablation of TIGDA for task C\textrightarrow A. 
(wEc=without Edge concatenation, wPl=without Perceptual loss, wLS=without Latent Search). The best source-like features are observed in the image with all the components of TIGDA.}
\label{fig:ablation_imgs}
\end{figure}
\setlength{\textfloatsep}{0pt}
\subsubsection{Effect of Perceptual loss $\mathcal{L}_{p}$} We have used a perceptual model $P_{\psi}$ trained on source samples while training the VAE.
Perceptual loss minimizes the Euclidean distance between the (perceptual) feature vectors of input and reconstructed source images. 
It measures image similarities more robustly than per-pixel losses (e.g., Mean squared error). It ensures that the VAE reconstructed image is semantically similar to the input. We can observe from Figure \ref{fig:wpl} that VAE reconstructed image without perceptual loss (wPl) during training, has different color and texture patterns from the real target image shown in Figure \ref{fig:ab_realt}. The finer background details are missing in Figure \ref{fig:wpl}. Such images will result in a poor latent space and the performance on target images will drop during inference. Table \ref{tab:ablation} shows that the accuracy drops to 53.4\% for the task A\textrightarrow B without perceptual loss while it is 76.2\% for TIGDA that uses perceptual loss during training. Similarly the accuracy drops to 57.1\% for the task B\textrightarrow C when perceptual loss was not employed during training but the accuracy on the same task is 74.8\% with perceptual loss.
\setlength{\textfloatsep}{0pt}
\begin{table}
\caption{Ablation of different components of TIGDA during training and inference; Edge, perceptual loss $\mathcal{L}_{p}$ and Latent Search (LS). Accuracy (mean $\pm$ std\%) values are reported as an average over five independent runs for two tasks.}
\begin{center}
\scalebox{1.0}{
\begin{tabular}{ccc|c|c}
    \toprule
    Edge & $\mathcal{L}_{p}$ & LS & {A\textrightarrow B} &  {B\textrightarrow C} \\
    \midrule
    &&&35.8$\pm$0.2&39.5$\pm$0.1\\
    \checkmark&&&39.7$\pm$0.4&42.2$\pm$0.3\\
    &\checkmark&&38.9$\pm$0.5&43.4$\pm$0.3\\
    &&\checkmark&50.2$\pm$0.3&52.8$\pm$0.2\\
    \checkmark&\checkmark&&43.7$\pm$0.2&46.9$\pm$0.5\\
    &\checkmark&\checkmark&57.6$\pm$0.4&60.3$\pm$0.2\\
    \checkmark&&\checkmark&53.4$\pm$0.3&57.1$\pm$0.4\\
    \checkmark&\checkmark&\checkmark&76.2$\pm$0.3&74.8$\pm$0.6\\
\bottomrule
\end{tabular}
}
\end{center}
\label{tab:ablation}
\end{table} 
\setlength{\textfloatsep}{0pt}
\subsubsection{Effect of Latent Search and other Loss functions}
To validate the importance of the Latent Search procedure, in Figure \ref{fig:wls} we show the VAE reconstructed images  without Latent Search for the target image shown in Figure \ref{fig:ab_realt}. 
Figure \ref{fig:ab_ours} shows the generated image after Latent Search for the task C\textrightarrow A. It is observed (empirically) that the `closest-clone' obtained through TIGDA shown in Figure \ref{fig:ab_ours} is visually more closer to the source domain  as compared to VAE reconstructed image shown in Figure \ref{fig:wls}. When no  Latent Search is employed, the accuracy for the tasks A\textrightarrow B and B\textrightarrow C drops to 43.7\% and 46.9\% respectively as shown in Table \ref{tab:ablation}. 
\setlength{\textfloatsep}{0pt}
\begin{figure*}
\centering
     \subfloat[Accu. vs. SSIM window sizes.]{\includegraphics[width=0.2434\linewidth,height=0.165\linewidth]{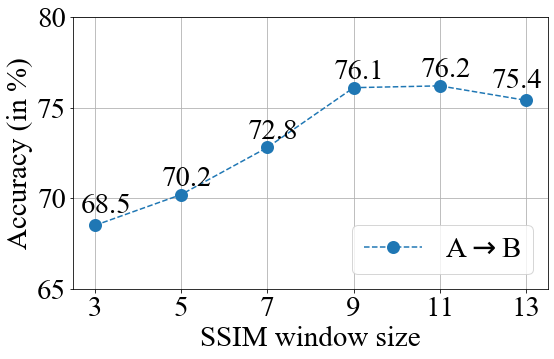}
      \label{fig:window}}
         \subfloat[Accu. vs. Edge concat. position.]{\includegraphics[width=0.2434\linewidth,height=0.165\linewidth]{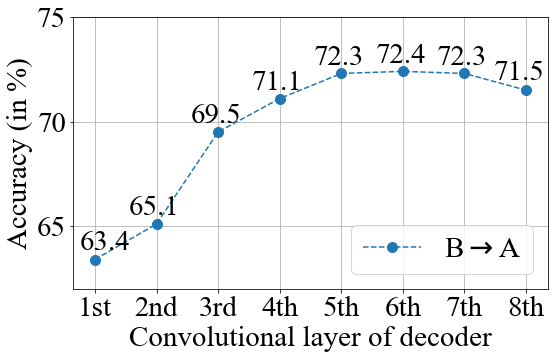}
      \label{fig:dge_sweep}}
    \subfloat[Accuracy vs Skip connections. ]{\includegraphics[width=0.2434\linewidth,height=0.165\linewidth]{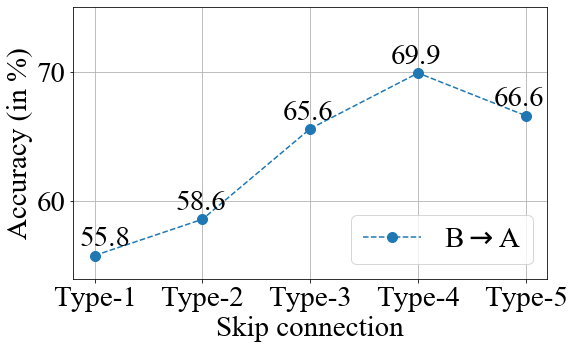}
      \label{fig:skip}} 
      \subfloat[FID vs. no. of training images.]{\includegraphics[width=0.2434\linewidth,height=0.165\linewidth]{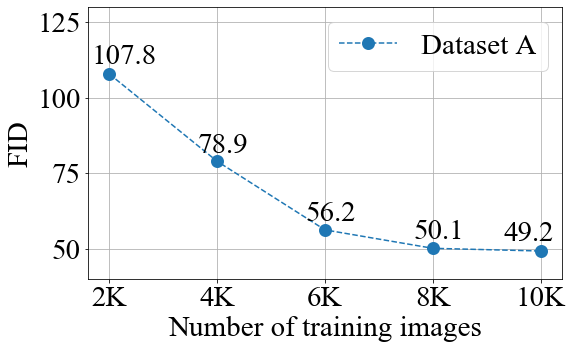}
        \label{fig:fid}}
 \caption{\bb{(a) Accuracy of TIGDA on task A\textrightarrow B by selecting different window sizes in SSIM during Latent Search (b) Performance of TIGDA when the edges of input images are concatenated with different convolutional layers in decoder $h_{\theta}$ (c) Performance of TIGDA when edge concatenation is replaced with different types of skip connections between encoder $g_{\phi}$ and decoder $h_{\theta}$ layers. Window size of 11 gives the best performance. For the same task, edge concatenation is better than skip connections. (d) FID of VAE generated images when TIGDA is trained on dataset A with different number of images ranging from 2,000 (2K) to 10,000 (10K). }}
\label{fig:f_w_e_s}
\end{figure*}
To affirm the usefulness of the choice of SSIM as loss for the Latent Search, we implemented Latent Search with three different losses, Mean Squared Error (MSE), Mean Absolute Error (MAE) and Structural Similarity Index (SSIM) loss and found that SSIM loss is the best performing among the three. SSIM loss compares pixels and their corresponding neighborhoods in two images, preserving the luminance, contrast and structure information. On the other hand, MSE or MAE measures only the absolute pixel differences rather than the structural differences. Figure \ref{fig:alossacc} and \ref{fig:blossacc} depict the outcome of these ablation studies where the superiority of the SSIM loss  is seen over MSE and MAE for the tasks A\textrightarrow C and B\textrightarrow C respectively. Table \ref{tab:ablation} summarizes all the ablation studies conducted  on two domain adaptation tasks with different combinations of the components. It can be noted that the best performance is observed by utilizing all the three components: Edge concatenation, perceptual loss and Latent Search procedure. Thus, with all the aforementioned studies, we have demonstrated the utility of all the individual components used in TIGDA for UDA task on WBC classification.  

\subsubsection{Effect of other hyperparameters} \bb{In this section, we study the effect of four hyperparameters: (a) the window size for the SSIM loss used for Latent Search, (b) the position of the Edge-operator in the decoder network, (c) use of Skip connection as in \cite{unet} instead of edge concatenation, (d) number of source samples required to generate high-fidelity images using VAE. Figure \ref{fig:f_w_e_s}(a) depicts the change in the performance for A\textrightarrow B with varying window sizes of SSIM. While the performance varies with different window sizes, the best accuracy is observed with the default choice of 11 that is used in all our experiments.}\par \bb{Next, in Figure \ref{fig:f_w_e_s}(b), we vary the layer of the decoder to concatenate edges. It is seen that the performance is best at the penultimate layers since the edges are used only to reduce the blurriness of the generated image that occurs near the last few layers of the decoder. Providing the edge information at initial layers of the decoder, regularizes more than required, thus degrading the quality of the generated image.} \par \bb{To further quantify the effect of edge concatenation as a regularizer, we replace it with another type of spatial contiguity in the form of skip connections as in a segmentation network such as UNet \cite{unet}.  We have used five different types of skip connections. Type-1 refers to no skip connection. Type-2 connects FC1 layer (Refer to Supplementary material for the names of the layers in the architecture) of the encoder with FC2 layer of the decoder network. Type-3 connects all the layers in the encoder with layers of corresponding dimensions in the decoder (like a U-Net). Type-4 connects Conv1 layer in the encoder with Conv9 layer of the decoder. Type-5 is combination of Type-2 and Type-4 skip connections. We observe in Figure \ref{fig:f_w_e_s}(c) that having skip connection is better than not having it since it regularizes the network. Further, Type-4, that connects the initial layers of the encoder with final layers of the decoder, has the best performance. This can be explained by the fact that initial layers of the CNNs are known to extract edge-like features which is shown to enhance the performance in the given task. Connecting more layers as in Type-3 and Type-5 leads to over regularization and degrades the performance. However, explicit edge concatenation still provides the best performance.} \par \bb{In the final plot Figure \ref{fig:f_w_e_s}(d), we report the Fréchet Inception Distance (FID)~\cite{heusel2017gans}, that quantifies the quality of the generated data (lower the better) for any generative model, as a function of the number of source samples used to train the VAE. It is seen that with the increase in number of images for training VAE, the quality of generated images improve as shown by the FID values. Therefore, with about 10K samples, one can expect the VAE to sample high-fidelity source images.}

\begin{table*}[t]
\caption{ \bb{Accuracy (mean $\pm$ std\%) values for UDA tasks on Office-31 and Imaging Flow Cytometry (Cyto.) and grayscale Peripheral Blood Smear (gray-PBS) White Blood Cell datasets. 
Results are reported as an average over five independent runs using various SOTA UDA methods using ResNet-50 classifier. Note that while all UDA methods perform better than the source only model, TIGDA offers significant performance enhancement despite not using the target images during training.}}
 \begin{center}
\color{black}\scalebox{0.8}{
\begin{tabular}{l|ccccccc|ccc}
    \toprule
  
    \multicolumn{1}{c|}{} & \multicolumn{7}{c|}{Office-31} & \multicolumn{3}{c}{WBC}\\
    Models & {A\textrightarrow W} &  {D\textrightarrow W} & {W\textrightarrow D} & {A\textrightarrow D} & {D\textrightarrow A} & {W\textrightarrow A} & {Avg} & {gray-PBS\textrightarrow Cyto.} &  {Cyto.\textrightarrow gray-PBS}  & {Avg}\\
    \midrule
    {Source Only}&68.4$\pm$0.2&96.7$\pm$0.1&99.3$\pm$0.1&68.9$\pm$0.2&62.5$\pm$0.3&60.7$\pm$0.3&76.1&42.6$\pm$0.1&22.2$\pm$0.2& 32.4\\

    {JAN~\cite{long2017deep}}&85.4$\pm$0.3&97.4$\pm$0.2&99.8$\pm$0.2&84.7$\pm$0.3&68.6$\pm$0.3&70.0$\pm$0.4&84.3&67.5$\pm$0.2&57.2$\pm$0.3&62.3\\
   
  
   {MADA~\cite{pei2018multi}}&90.0$\pm$0.2&97.4$\pm$0.1&99.6$\pm$0.1&87.8$\pm$0.2&70.3$\pm$0.4&66.3$\pm$0.1&85.2&73.3$\pm$0.2&61.8$\pm$0.3&67.5\\
   

   {SimNet~\cite{pinheiro2018unsupervised}}&88.6$\pm$0.5&98.2$\pm$0.2&99.7$\pm$0.2&85.3$\pm$0.3&73.4$\pm$0.8&71.8$\pm$0.6&86.2&76.4$\pm$0.2&66.8$\pm$0.2&71.6\\

   {GTA~\cite{sankaranarayanan2018generate}}&89.5$\pm$0.5&97.9$\pm$0.3&99.8$\pm$0.4&87.7$\pm$0.5&72.8$\pm$0.3&71.4$\pm$0.4&86.5&75.2$\pm$0.4&66.5$\pm$0.3&70.8\\

    {DAAA~\cite{kang2018deep}}&86.8$\pm$0.2&99.3$\pm$0.1&100.0$\pm$0.0&88.8$\pm$0.4&74.3$\pm$0.2&73.9$\pm$0.2&87.2&75.8$\pm$0.3&68.2$\pm$0.1&72.0\\
    

    {CDAN~\cite{long2018conditional}}&94.1$\pm$0.1&98.6$\pm$0.1&100.0$\pm$0.0&92.9$\pm$0.2&71.0$\pm$0.3&69.3$\pm$0.3&87.7&78.6$\pm$0.2&67.1$\pm$0.1&72.8\\
   


    {CAN~\cite{kang2019contrastive}}&94.5$\pm$0.3&99.1$\pm$0.2&99.8$\pm$0.2&95.0$\pm$0.3&78.0$\pm$0.3&77.0$\pm$0.3&90.6&79.4$\pm$0.3&68.9$\pm$0.2&74.1\\




  
  {TIGDA (Ours)}&93.2$\pm$0.2&99.4$\pm$0.4&99.8$\pm$0.1&93.6$\pm$0.3&76.7$\pm$0.2&75.7$\pm$0.3&89.7&80.3$\pm$0.4&71.4$\pm$0.3&75.8\\
    
    \bottomrule
\end{tabular}
}
\end{center}
\label{tab:compoffice}
\end{table*}

\section{TIGDA beyond PBS}
\bb{
In this section, we examine the effectiveness of the proposed method TIGDA on two datasets, Imaging Flow Cytometry~\cite{lippeveld2019classification} and Office-31~\cite{saenko2010adapting}, apart from PBS. In Cytometry dataset, WBCs from whole blood samples were stained using a ImageStream-X MK II imaging flow cytometer. A three channel image is extracted with two bright-field (at wavelengths of 420 nm - 480 nm and 570 nm - 595 nm) and a dark-field channel. Four classes of WBCs are employed in this study: Eosinophil (1470 images), Neutrophil (4809 images), Lymphocyte (4570 images) and Monocyte (1239 images). The objective of this experiment is to examine if TIGDA can perform domain adaptation when the source is Cytometry data and the target is PBS and vice versa. Since Cytometry data doesn't have the notion of color, we take the grayscale version of the PBS dataset with a $60\times60$ central crop (in all the images) representing the nucleus. Figure \ref{fig:cytofig} depicts a sample image from each class of the Cytometry dataset and the PBS dataset which apparently shows a significant domain shift.} \bb{Office-31~\cite{saenko2010adapting}, a publicly available standard dataset for UDA tasks (some sample images are given in the supplementary material), contains images from 31 common object types taken with three different imaging sources namely Dslr (D), Webcam (W) and Amazon (A). The objective of UDA is to adapt between these three domains.}

\bb{
Table \ref{tab:compoffice} lists the results of TIGDA along with some SOTA UDA methods for domain adaptation tasks on both the Office-31 and Cytometry datasets. It is seen that on Cytometry and gray-PBS datasets, TIGDA performs the best by significantly improving upon the Source Only model for gray-PBS\textrightarrow Cyto. and Cyto.\textrightarrow gray-PBS tasks. Whereas, on the Office-31 dataset, TIGDA's average performance is comparable (less than a percent) to the best SOTA method. All these experiments firmly demonstrate the effectiveness of TIGDA in UDA despite not using the target data during training.}
\begin{figure}[!t]
\includegraphics[width=.48\textwidth,height=.236\textwidth]{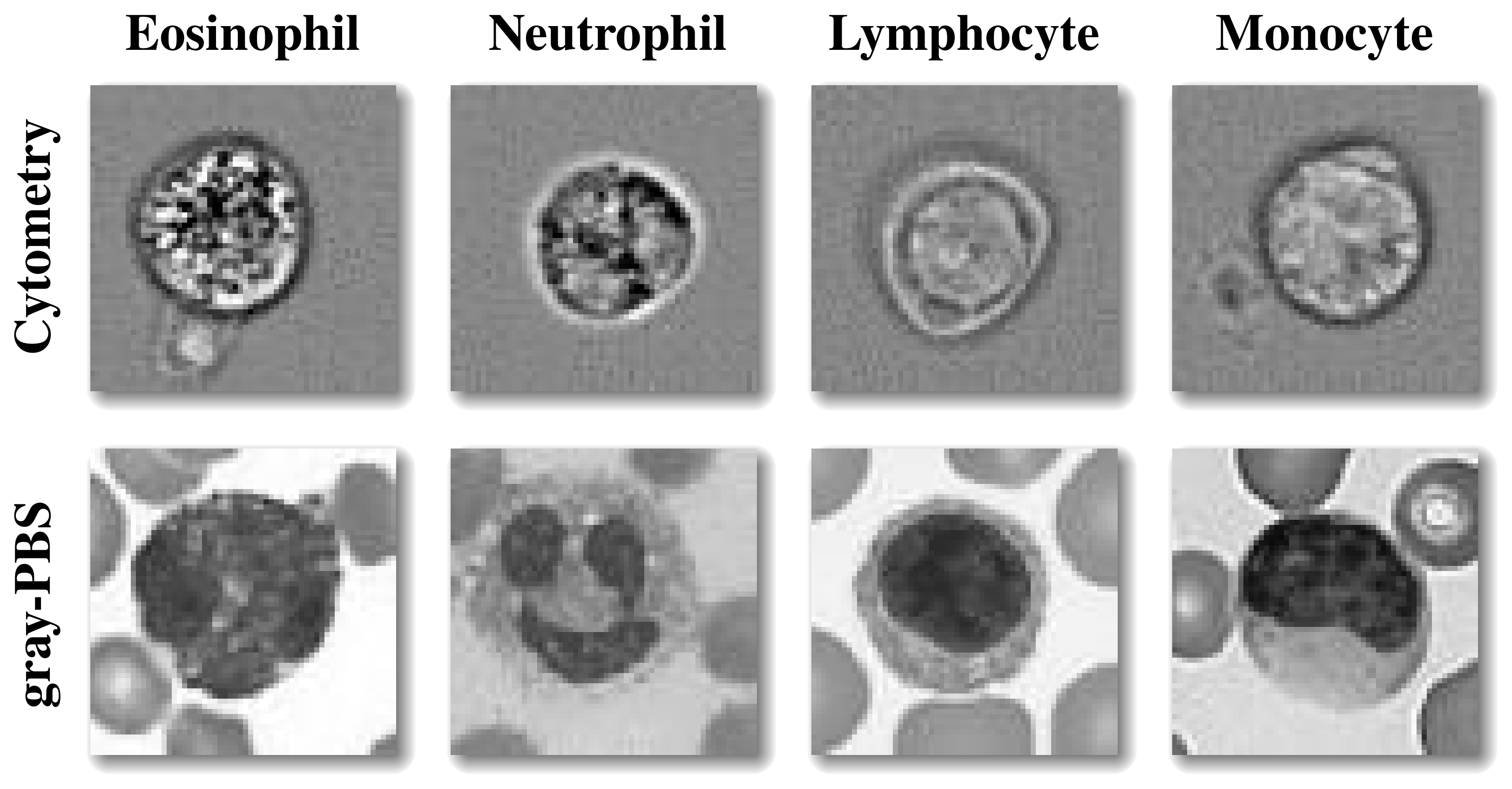}
\caption{\bb{Imaging Flow Cytometry~\cite{lippeveld2019classification} and grayscale Peripheral Blood Smear (gray-PBS) White Blood Cell datasets.}}
\label{fig:cytofig}
\end{figure}
\section{Conclusion}
\bb{In this work, we have considered the problem of domain shift occurring with the CNN-based classifiers for WBC classification. The performance of the existing \cc{deep learning} based techniques is known to degrade with the change in camera characteristics. We cast the problem of performance degradation of WBC classifiers with the change in camera as that of \cc{Unsupervised Domain Adaptation} (UDA) and propose a method that is devoid of need for access to the target data during training. We have demonstrated the efficacy of the proposed method for UDA with experiments on multiple datasets acquired under different settings. A few possible future directions can be: (i) extension of TIGDA for medical data beyond WBC, (ii) combining multiple sources for UDA.}

\section{Acknowledgements}
\bb{We sincerely thank the Associate Editor and the Anonymous Reviewers for their thoughtful comments that helped to significantly improve our paper. We also thank Maxim Lippeveld, Ghent University for his generous help in providing and navigating through the Cytometry dataset. We thank Sameer Ambekar and Aayush Tyagi for their help in experiments.} 

\bibliographystyle{ieeetr}
\bibliography{bibtex}
%






\clearpage

\title{Target-Independent Domain Adaptation for WBC Classification using Generative Latent Search \\--Supplementary--}
 \author{
Prashant Pandey, Prathosh AP, Vinay Kyatham, Deepak Mishra and 
Tathagato Rai Dastidar
}
\maketitle
\setlength{\dbltextfloatsep}{0.2cm}
\section{Proof for Lemma 1}:
\begin{lemma}

If $\tilde{\mathbf{x}}_\mathcal{S} \in \mathcal{S}_n$ is the point such that $\mathfrak{D\{\tilde{\mathbf{x}}_\mathcal{T}},\tilde{{\mathbf{x}}}_\mathcal{S}\}< \mathfrak{D\{\tilde{\mathbf{x}}_\mathcal{T}},\mathbf{x} \} \ \forall \ \mathbf{x}\ \in\ \mathcal{S}_n $, then as ${n\to\infty}$,\  $\tilde{\mathbf{x}}_\mathcal{S}$ converges to  $\tilde{\mathbf{x}}_\mathcal{T}$ with probability $1$. 
\end{lemma}

\begin{proof}
Let $\mathbb{B}_r({\tilde{\mathbf{x}}_\mathcal{T}})$  be a closed ball of radius $r$ around $\tilde{\mathbf{x}}_\mathcal{T}$ under the metric $\mathfrak{D}$. That is,  $\mathbb{B}_r({\tilde{\mathbf{x}}_\mathcal{T}}) = \{\mathbf{x}:\mathfrak{D\{\tilde{\mathbf{x}}_\mathcal{T}},\mathbf{x} \}\leq r\}$. Since $\mathfrak{X}$ is a separable metric space, $\forall r>0$, $\mathbb{B}_r({\tilde{\mathbf{x}}_\mathcal{T}})$ has non-zero probability measure [40]. That is, 
\begin{equation}\mathbf{Pr}\big(\mathbb{B}_r({\tilde{\mathbf{x}}_\mathcal{T}})\big) \triangleq \int\limits_{\mathbb{B}_r({\tilde{\mathbf{x}}_\mathcal{T}})}\mathcal{P}_s(\mathbf{x}) \ d\mathbf{x} > 0
\end{equation}
For any $\delta>0$ , the probability that  none of the points in $\mathcal{S}_n$ are within the ball $\mathbb{B}_\delta({\tilde{\mathbf{x}}_\mathcal{T}})$ of radius $\delta$ is given by:
\setlength{\textfloatsep}{1pt}
\begin{equation}
\mathbf{Pr}\bigg[ \min \limits_{i=1,2..,n} \mathfrak{D\{\mathbf{x}_i,\tilde{\mathbf{x}}_\mathcal{T}}\} \geq \delta  \bigg] = \big[ 1- \mathbf{Pr}\big(\mathbb{B}_\delta({\tilde{\mathbf{x}}_\mathcal{T}})\big) \big]^n
\end{equation}
Therefore, the probability of $\tilde{{\mathbf{x}}}_\mathcal{S} \in  \mathcal{S}_n$, lying within $\mathbb{B}_\delta({\tilde{\mathbf{x}}_\mathcal{T}})$ is given by: 
\setlength{\textfloatsep}{1pt}
\begin{align}
\mathbf{Pr}\bigg[ \tilde{{\mathbf{x}}}_\mathcal{S} \in  \mathbb{B}_\delta({\tilde{\mathbf{x}}_\mathcal{T}})  \bigg] &= 1 -\big[ 1- \mathbf{Pr}\big(\mathbb{B}_\delta({\tilde{\mathbf{x}}_\mathcal{T}})\big) \big]^n\\
&= 1\ \ as \ n\rightarrow \infty
\end{align}
Thus, given any $\delta>0$, with probability $1$, $\exists \  \tilde{{\mathbf{x}}}_\mathcal{S} \in \mathcal{S}_n$ that is within  $\delta$ distance from $\tilde{\mathbf{x}}_\mathcal{T}$ as $n\rightarrow \infty$
\end{proof}

\vspace{2.9 in}
\section{Algorithm for TIGDA}

\begin{algorithm}
    \caption{\textbf{Target-Independent Generative Domain Adaptation (TIGDA)}}
    \label{alg:trainalgo}
    
    \hspace*{\algorithmicindent} 
    
    \begin{algorithmic}[1] 
    \phase{{Training VAE on source data}}
    \textbf{Input}: Source dataset $\mathcal{S}_{n}= \{\rvx_{1},..., \rvx_{n} \}$, Number of source images $n$, Encoder $g_{\phi}$, Decoder $h_{\theta}$, Trained Perceptual Model ${P}_{\psi}$, Learning rate $\eta$, Batchsize $B$. \textbf{Output}: Optimal parameters $\phi^*$, $\theta^*$.
    \State Initialize parameters $\phi$, $\theta$
    \REPEAT
    \State sample batch $\{\rvx_{i}\}$ from dataset $\mathcal{S}_{n},$ for $ i = 1,..., B$  
    \State $\mu_\rvz^{(i)}, \sigma_\rvz^{(i)} \gets g_{\phi} ( \rvx_{i})$
    \State sample ${\rvz}_{i} \sim \mathcal{N}(\mu_\rvz^{(i)},$ ${\sigma_\rvz^{(i)}}^2)$
    \State $\hat{\rvx}_{i} \gets h_{\theta}({\rvz}_{i})$ 
    \State $ \mathcal{L}_{r} \gets \sum_{i=1}^{B} \left\Vert \rvx_{i} - \hat{\rvx}_{i}\right\Vert_2^2$ 
    \vspace*{0.08cm}
    \State $ \mathcal{L}_{p} \gets \sum_{i=1}^{B} \left\Vert P_{\psi}( \rvx_{i}) - P_{\psi}(\hat{\rvx}_{i})\right\Vert_2^2 $
    
    \State $ \mathcal{L}_{g}\gets\mathcal{L}_{r} + \mathcal{L}_{p} + \sum_{i=1}^{B} \mathbb{D}_{KL} \left[\mathcal{N}(\mu_\rvz^{(i)},{\sigma_\rvz^{(i)}}^2 ) \,|| \,\mathcal{N}(0, 1) \right]$
    \State $ \mathcal{L}_{h} \gets  \mathcal{L}_{r} + \mathcal{L}_{p}$ 
    \State $ \phi \gets \phi + \eta \nabla_{\phi}\mathcal{L}_{g} $ 
    \State $ \theta \gets \theta + \eta \nabla_{\theta}\mathcal{L}_{h}$ 
    \UNTIL{convergence of $\phi$, $\theta$ }
    \phase{Inference - Latent Search for target images }
     \textbf{Input}: Target image $\rm \tilde{\rvx}_\mathcal{T}$, Trained decoder $h_{\theta^{\ast}}$, Learning rate $\eta$. \textbf{Output}: `closest-clone' $\tilde{{\mathbf{x}}}_\mathcal{S}$  for the target image $\rm \tilde{\rvx}_\mathcal{T}$.  
    \State sample ${\mathbf{z}}$ from $\mathcal{N}(0,1)$
    \REPEAT
    
    \State $ \mathcal{L}_{ssim} \gets 1 - \text{SSIM}({\rm \tilde{\rvx}}_\mathcal{T}, h_{\theta^*}({\rvz}))$
    \State $ {\rvz} \gets {\rvz}  + \eta \nabla_{\rvz}\mathcal{L}_{ssim} $
    \UNTIL{convergence of  $ \mathcal{L}_{ssim}$}
    \State $ \rm \tilde{\rvz}_\mathcal{S} \gets {\rvz}$
    
    \State $  \tilde{\rvx}_\mathcal{S} \gets h_{\theta^*} (\tilde{\mathbf{z}}_\mathcal{S})$
    \end{algorithmic}
\end{algorithm}


\begin{figure*}[th]
\centering
    \subfloat[ADDA]{\includegraphics[width=0.16\linewidth,height=0.14\linewidth]{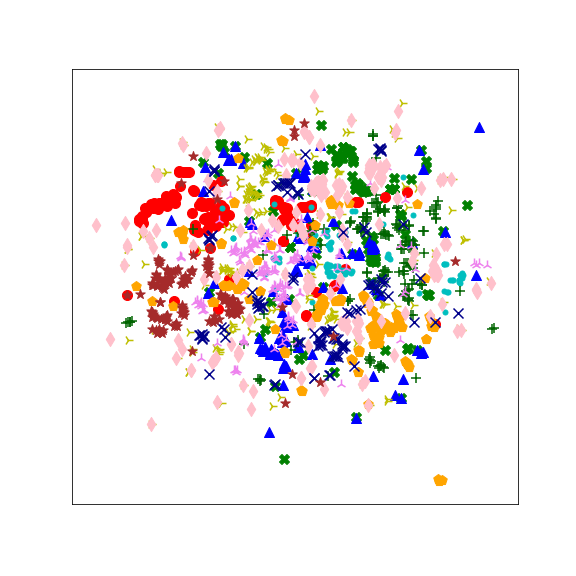}
        \label{fig:adda}}
     \subfloat[GTA]{\includegraphics[width=0.16\linewidth,height=0.14\linewidth]{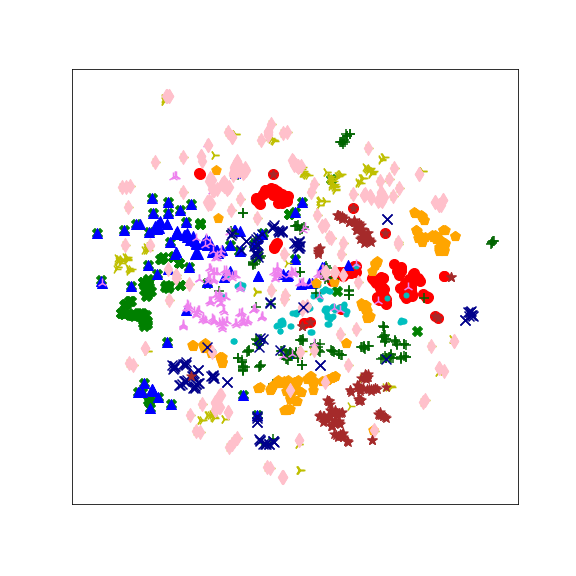}
      \label{fig:gta}}
         \subfloat[DIRT-T]{\includegraphics[width=0.16\linewidth,height=0.14\linewidth]{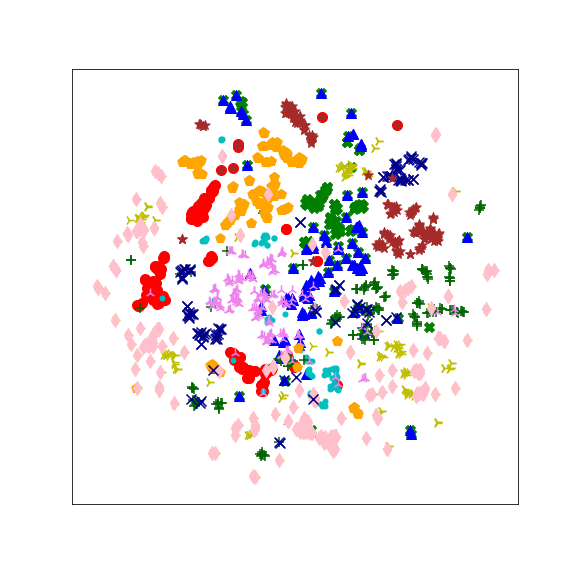}
      \label{fig:dirtt}}
    \subfloat[TAT]{\includegraphics[width=0.16\linewidth,height=0.14\linewidth]{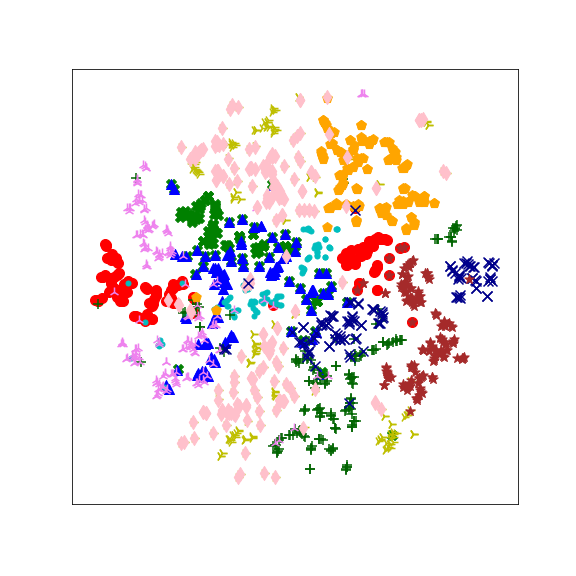}
      \label{fig:tat}} 
     \subfloat[DAL]{\includegraphics[width=0.16\linewidth,height=0.14\linewidth]{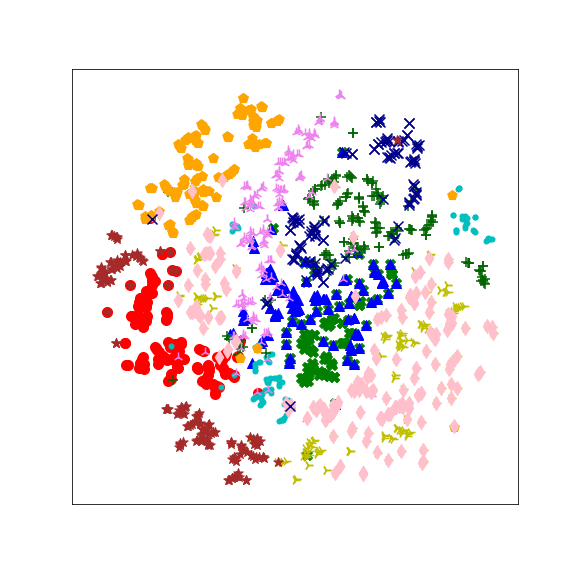}
      \label{fig:dal}} 
     \subfloat[TIGDA]{\includegraphics[width=0.16\linewidth,height=0.14\linewidth]{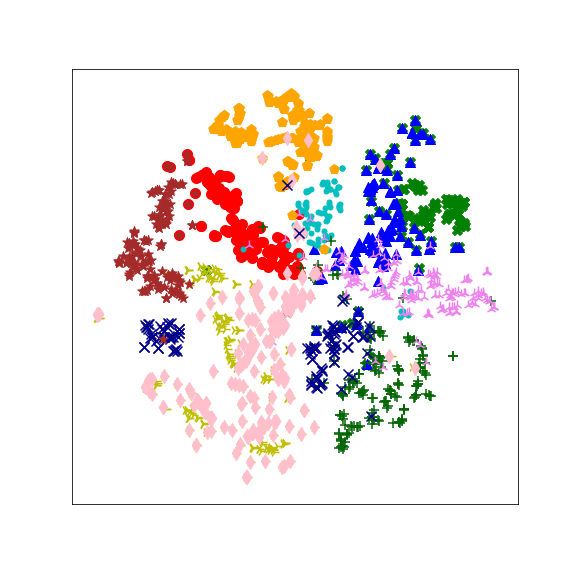}
      \label{fig:ours_tsne}} 
 \caption{PCA plots of the first two principal component using features generated by ADDA, GTA, DIRT-T, TAT, DAL and TIGDA on task A\textrightarrow C.}
\label{fig:pca}
\end{figure*}

\begin{figure*}[th]
\centering
    \subfloat[ADDA]{\includegraphics[width=0.30\linewidth,height=0.30\linewidth]{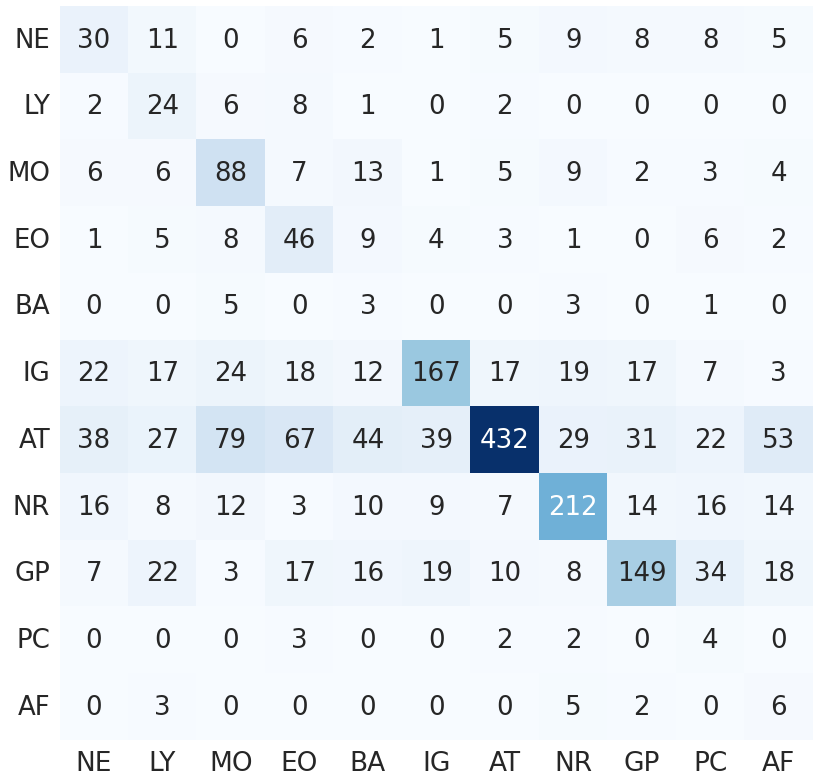}
        \label{fig:adda}}
     \subfloat[GTA]{\includegraphics[width=0.30\linewidth,height=0.30\linewidth]{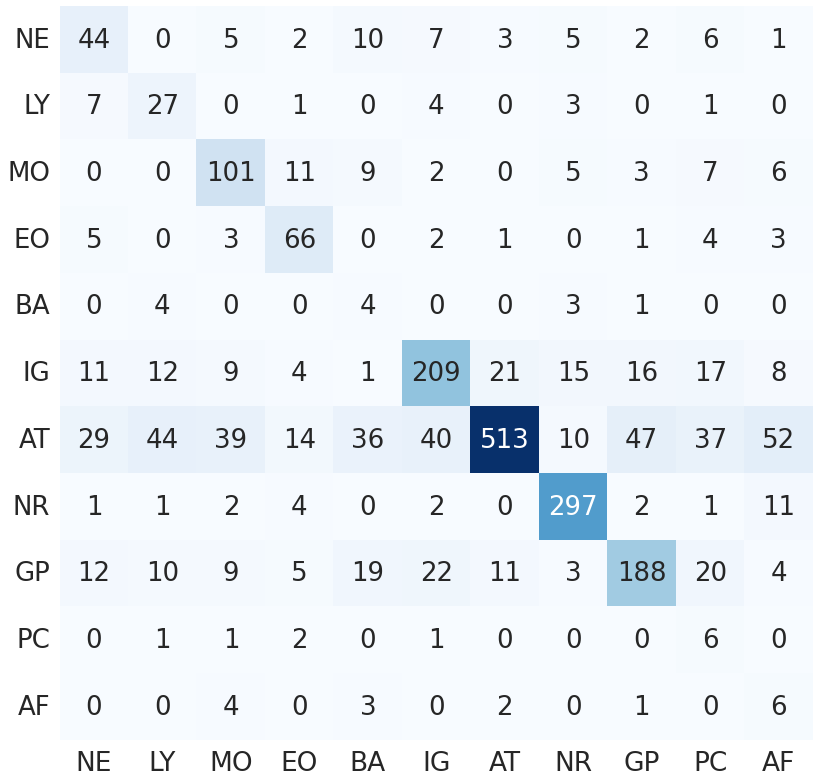}
      \label{fig:gta}}
         \subfloat[DIRT-T]{\includegraphics[width=0.30\linewidth,height=0.30\linewidth]{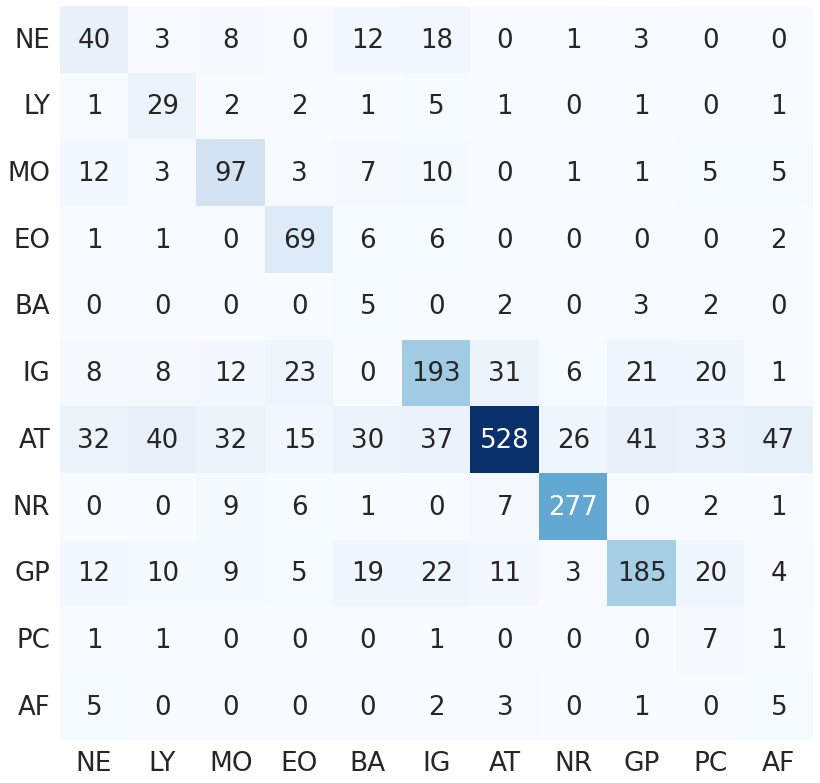}
      \label{fig:dirtt}}
      \\
    \subfloat[TAT]{\includegraphics[width=0.30\linewidth,height=0.30\linewidth]{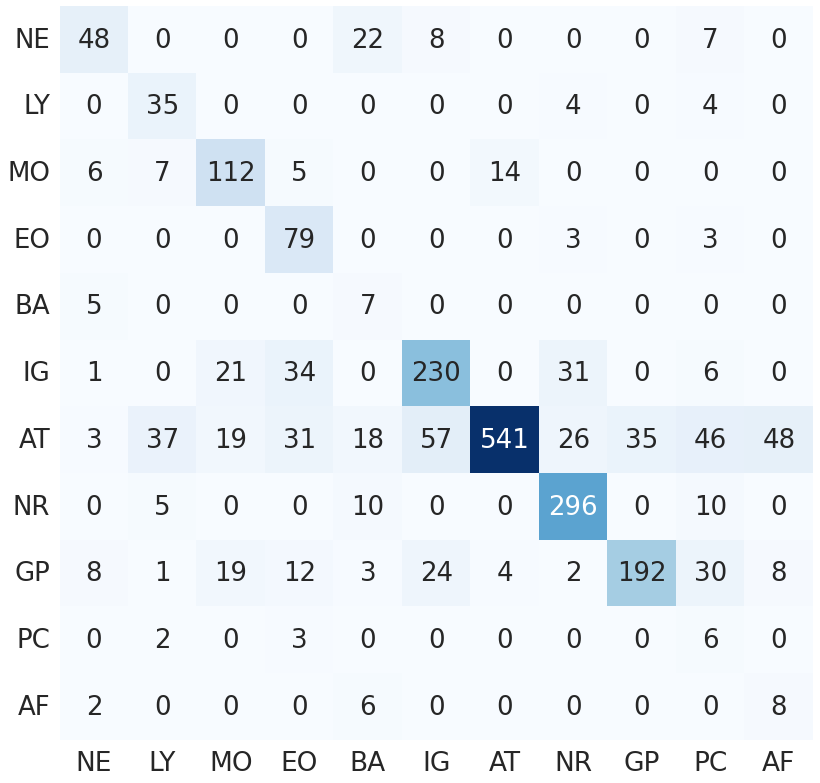}
      \label{fig:tat}} 
     \subfloat[DAL]{\includegraphics[width=0.30\linewidth,height=0.30\linewidth]{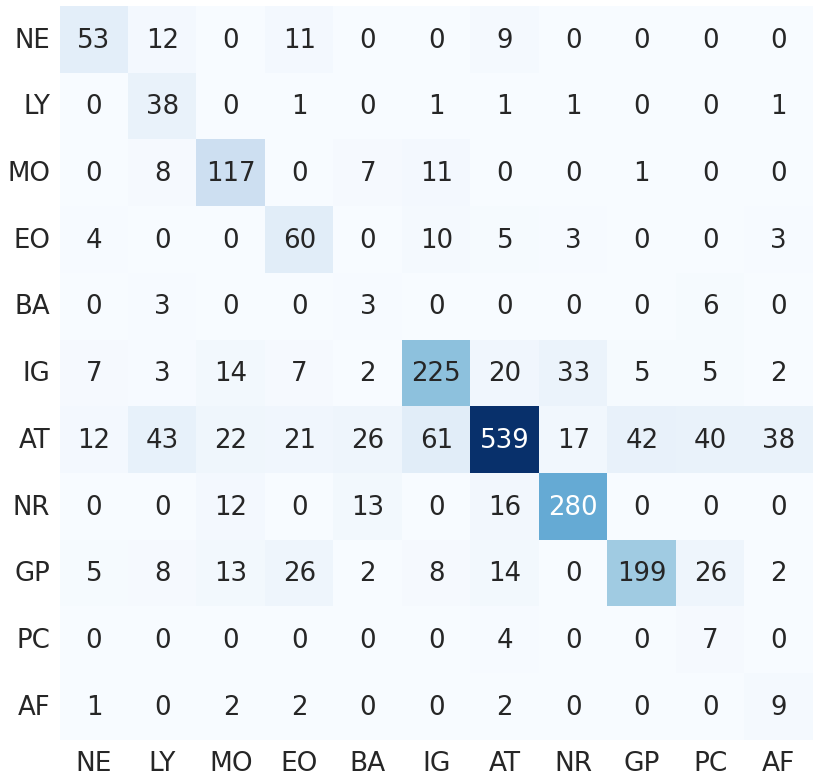}
      \label{fig:dal}} 
     \subfloat[TIGDA]{\includegraphics[width=0.30\linewidth,height=0.30\linewidth]{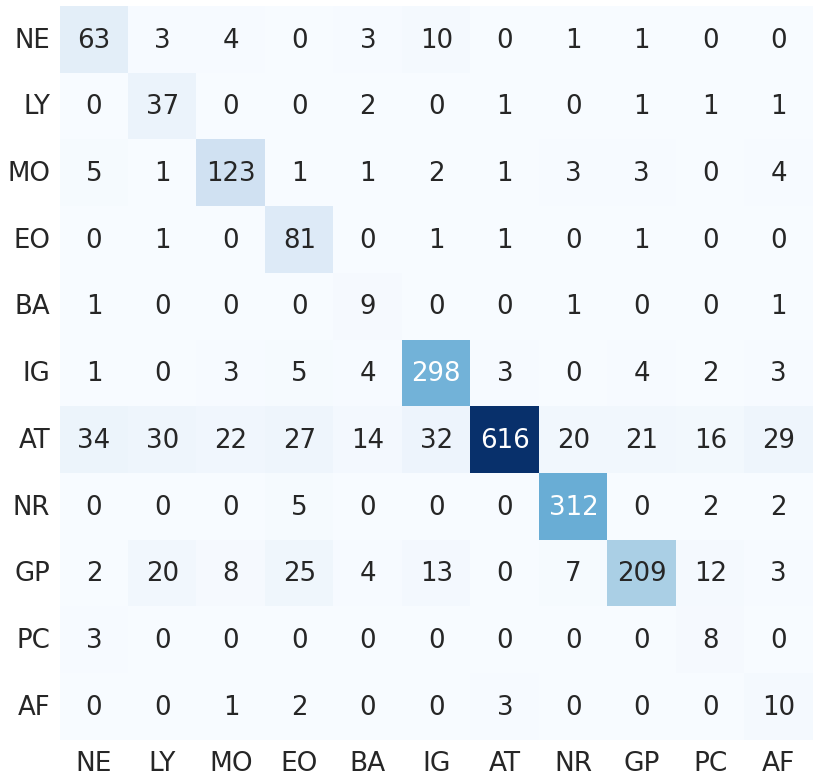}
      \label{fig:ours_tsne}} 
 \caption{Confusion Matrices for ADDA, GTA, DIRT-T, TAT, DAL and TIGDA on task A\textrightarrow C. Classes are Neutrophil (NE), Lymphocyte (LY), Monocyte (MO), Eosinophil (EO), Basophil (BA), Immature granulocytes (IG), Atypical (AT), Nucleated red blood cells (NR), Giant platelets (GP), Platelet clumps (PC), Artefact (AF).}
\label{fig:pca}
\end{figure*}

\begin{table*}
\caption{Encoder architecture for the Variational Auto-Encoder (VAE) in the proposed method (TIGDA). Convolution kernel is $3\times3$ and for Leaky ReLU $\alpha=0.2$.}
\centering
  \begin{tabular}{cc}
    \toprule
        {Layer (type)} & Output shape 
        \\
    \midrule
    encoder\_input (InputLayer) & (128, 128, 3) \\
    Conv1 (Convolution) & (128, 128, 128) \\
    leakyReLU1 (Activation) & (128, 128, 128) \\
    Conv2 (Convolution) & (64, 64, 128) \\
    leakyReLU2 (Activation) & (64, 64, 128)\\
    Conv3 (Convolution) & (32, 32, 128) \\
    leakyReLU3 (Activation) & (32, 32, 128) \\
    Conv4 (Convolution) & (16, 16, 128) \\
    leakyReLU4 (Activation) & (16, 16, 128) \\
    Conv5 (Convolution) & (8, 8, 128) \\
    leakyReLU5 (Activation) & (8, 8, 128) \\
    Conv6 (Convolution) & (4, 4, 128) \\
    leakyReLU6 (Activation) & (4, 4, 128) \\
    FC1 (Dense) & (1024) \\
    Z (Dense) & (64) \\
    \bottomrule
  \end{tabular}
  \label{table:enc}
\end{table*}

\begin{table*}
\caption{Decoder architecture for the VAE in TIGDA. Convolution kernel is $3\times3$ and for Leaky ReLU $\alpha=0.2$.}
\centering
  \begin{tabular}{cc}
    \toprule
        {Layer (type)} & Output shape 
        \\
    \midrule
    decoder\_input (InputLayer) & (64) \\
    FC2 (Dense) & (1024) \\
    leakyReLU7 (Activation) & (1024) \\
    FC3 (Dense) & (2048) \\
    leakyReLU8 (Activation) & (2048) \\
    Deconv1 (Deconvolution) & (4, 4, 128) \\
    leakyReLU9 (Activation) & (4, 4, 128) \\
     Deconv2 (Deconvolution) & (8, 8, 128) \\
     leakyReLU10 (Activation) & (8, 8, 128) \\
     Deconv3 (Deconvolution) & (16, 16, 128) \\
     leakyReLU11 (Activation) & (16, 16, 128) \\
     Deconv4 (Deconvolution) & (32, 32, 128) \\
     leakyReLU12 (Activation) & (32, 32, 128) \\
     Deconv5 (Deconvolution) & (64, 64, 128) \\
     leakyReLU13 (Activation) & (64, 64, 128) \\
     Deconv6 (Deconvolution) & (128, 128, 3) \\
     tanh1 (Activation) & (128, 128, 3) \\
     edge\_input (InputLayer) & (128, 128, 6)\\
     edge\_concat (Concatenate) & (128, 128, 9)\\
  Conv7 (Convolution) & (128, 128, 128) \\
  leakyReLU14 (Activation) & (128, 128, 128) \\
    Conv8 (Convolution) & (128, 128, 128) \\
    leakyReLU15 (Activation) & (128, 128, 128) \\
    Conv9 (Convolution) & (128, 128, 3) \\
    tanh2 (Activation) & (128, 128, 3) \\
    \bottomrule
  \end{tabular}
  \label{table:enc}
\end{table*}

\begin{figure*}[t]

\centering
 \includegraphics[width=0.81\textwidth,height=0.294\textwidth]{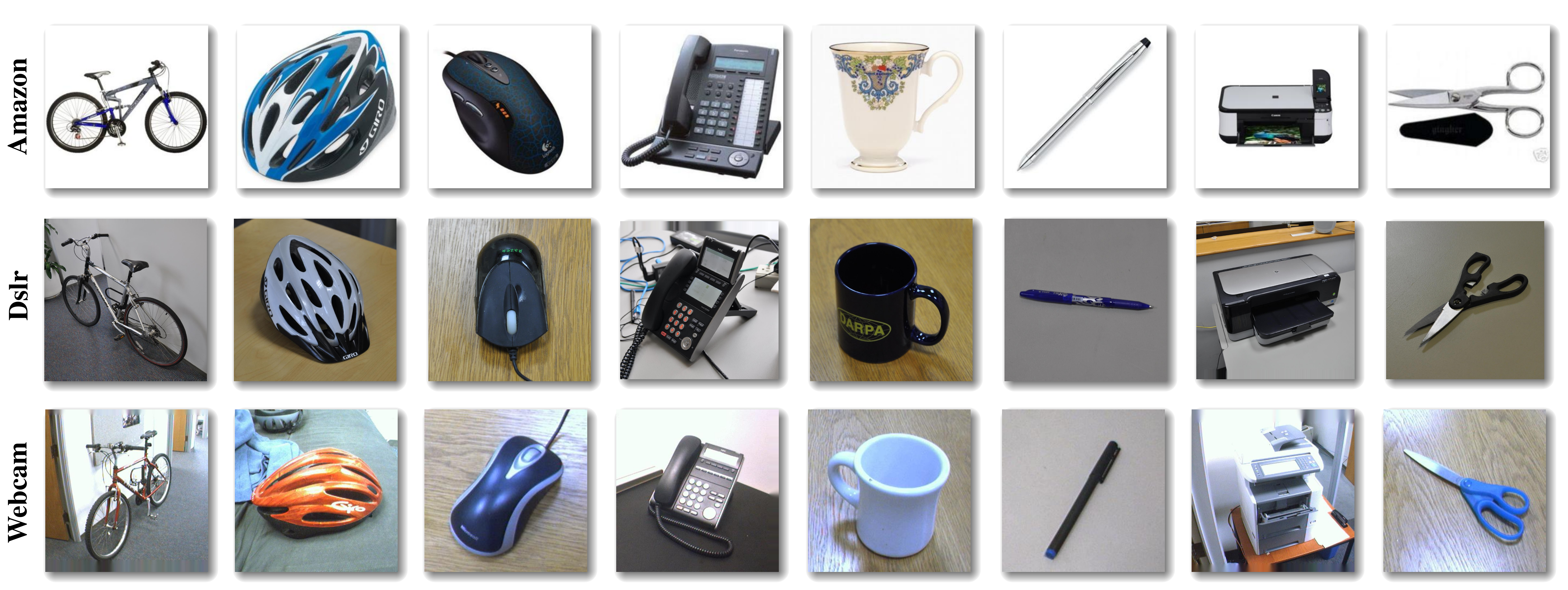}
\caption{Samples from the Office-31 dataset from the three sources, Amazon, Dslr and Webcam.}
\label{fig:office31}
\end{figure*}

\end{document}

%% file: math_commands.tex
\usepackage{amsmath,amsfonts,bm}

\def\eqref#1{equation~\ref{#1}}

\def\1{\bm{1}}

\def\rvx{{\mathbf{x}}}

\def\rvz{{\mathbf{z}}}

\DeclareMathAlphabet{\mathsfit}{\encodingdefault}{\sfdefault}{m}{sl}
\SetMathAlphabet{\mathsfit}{bold}{\encodingdefault}{\sfdefault}{bx}{n}

